\documentclass[journal]{IEEEtran}

\usepackage{amsmath,amssymb,amsfonts}
\usepackage{amsthm}
\usepackage{graphicx}
\usepackage{booktabs}
\usepackage{bm}
\usepackage{cite}
\usepackage{enumitem}
\usepackage[colorlinks=true,linkcolor=blue,citecolor=blue,urlcolor=blue]{hyperref}

\newtheorem{theorem}{Theorem}
\newtheorem{proposition}{Proposition}
\newtheorem{corollary}{Corollary}
\newtheorem{remark}{Remark}
\newtheorem{definition}{Definition}

\newcommand{\Lam}{\Lambda(q)}
\newcommand{\Laminv}{\Lambda^{-1}(q)}
\newcommand{\Fmpc}{F_{\text{mpc}}}
\newcommand{\Fmpck}[1]{F_{\text{mpc},#1}}
\newcommand{\dhat}{\hat{d}}
\newcommand{\Jvb}{\bar{J}_v}
\newcommand{\Nbar}{\bar{N}}

\let\OLDthebibliography\thebibliography
\renewcommand\thebibliography[1]{%
  \footnotesize
  \OLDthebibliography{#1}%
  \setlength{\itemsep}{-1.5pt}%
  \setlength{\parsep}{0pt}%
  \setlength{\topsep}{0pt}%
}

\begin{document}

\title{Toward Interaction Dynamics: A Predictive Framework for Safe Physical Human Robot Interaction}

\author{Yongyan~Cao$^{1}$ and Jinshan~Tang$^{2}$%
\thanks{$^{1}$Y. Cao is with Voryx Robotics LLC, San Jose, CA 95136, USA
(e-mail: yongyancao@gmail.com).}%
\thanks{$^{2}$J. Tang is with the Department of Health Administration and
Policy, George Mason University, Fairfax, VA, USA
(e-mail: jtang25@gmu.edu).}}

\maketitle

\begin{abstract}
Physical human--robot interaction requires yielding transiently to contact yet
recovering the commanded reference under sustained load.  Finite-stiffness
impedance control retains a static deflection there, while predictive
alternatives typically optimize a nonlinear robot or impedance model online.
Operational-space cancellation instead exposes a translational error
\textbf{double integrator with a fixed transition matrix and a
configuration-scheduled input map}, making interaction a predictive quantity
rather than a property re-derived per configuration.  We
build on it a compact offset-free interaction-error MPC for torque-controlled
manipulators: a force-domain random-walk state estimates persistent interaction
and model error, and a 30-variable convex QP maps the correction through the
current task inertia while constraining the applied joint torque.  Conditional
results establish impedance equivalence of the unconstrained passive feedback,
offset-free regulation at feasible frozen configurations, and quadratic
stabilizability of the scheduled backbone.  In a 1\,kHz MuJoCo simulation of a
7-DOF Franka FR3, the estimator cuts steady-state error under a repeated 15\,N
step from 2.73\,mm to 0.033\,mm when added to the otherwise identical 100\,Hz
MPC, at unchanged peak joint power.  A stiffness-and-damping-calibrated impedance
baseline attains 2.63\,mm at the same actuator cost.  Adding
ideal measured-force cancellation to that baseline gives 1.36\,mm, so
constant-load rejection is not unique to MPC; the sensorless controller still
reaches 0.033\,mm in the moving task without force sensing.  The torque constraint
holds the command inside a budget tightened to the lowest the task can meet, at no
tracking cost.  On a physical
UR5e, five alternated pairs under a nominal 1.5\,kg hanging load give
1.505\,mm without the estimator; with it the controller is offset-free, its
0.023\,mm error being the noise level of the encoder-based position measurement.
A tuned anti-windup integral baseline also removes the offset but with larger
transients.  The contribution is an efficient operational-space realization
that composes with broader interaction-control architectures.
\end{abstract}


\section{Introduction}

Impedance and admittance control provide transparent contact behavior for
physical human--robot interaction (pHRI)~\cite{hogan1985impedance,chiaverini1999survey}.
Their intended objectives differ across applications.  In hand guiding, a
force-induced displacement may be the command itself; zero tracking error is
then undesirable.  This work instead addresses \emph{reference-holding and
recovery} tasks in which transient deflection is permitted but a robot should
return to a commanded motion under a persistent load.  With finite stiffness,
classical impedance retains the equilibrium bias
$p-p_d=K_d^{-1}F_h$.  Integral action can remove it but must be coordinated
with actuator saturation~\cite{cao2004antiwindup,cao2002antiwindup}.

Predictive interaction control already has a substantial literature. Gold
\emph{et al.} formulate model-predictive interaction control for manipulation
tasks~\cite{gold2023mpic}; Chen \emph{et al.} integrate impedance matching and
constraint optimization~\cite{chen2025unified}; Haninger \emph{et al.} learn
interaction models with Gaussian processes~\cite{haninger2023model}; and
Cao \emph{et al.} impose passivity in predictive impedance
control~\cite{cao2023passive}.  Online force estimation has also been combined
with MPC~\cite{jordana2024force}, while nonlinear manipulator MPC can run at
high rate with tailored computation~\cite{kleff2021highfrequency}.  Safe
learning-supported predictive force/motion control~\cite{matschek2023safe},
predictive variable impedance~\cite{liu2025model}, and MPC--ADRC
combinations~\cite{wu2025ensuring} address adjacent robustness and constraint
questions.  The present work builds on all three ingredients --- MPC,
disturbance estimation, and interaction-error prediction --- and contributes
their combination into a compact operational-space interface.

Two of these differ from the present formulation in model scope and sensing.
Jordana \emph{et al.}~\cite{jordana2024force} embed direct force-sensor
feedback in a full nonlinear whole-body MPC, estimating the discrepancy
between predicted and measured contact force within that model; the present
controller instead uses no force/torque sensor and estimates an unmodelled
aggregate disturbance as a random walk in a reduced, configuration-independent
three-dimensional force channel. Gold \emph{et al.}'s MPIC~\cite{gold2023mpic}
predicts contact force explicitly through an interaction model, weighting the
prediction independently in the cost alongside position, velocity, and
control-effort terms; for a reference-holding task, that weight must be zero
to avoid trading position accuracy for compliance, at which point the force
prediction stops contributing to disturbance rejection at all. The present
controller instead treats the disturbance as unknown, and its centered input
cost achieves offset rejection with no such trade-off: $\Fmpc=-\dhat$ is
exactly cost-free at any position weight, regardless of whether the task is
framed as motion, compliance, or force control.

The contribution is therefore the specific combination of a predictive
move and a force-domain disturbance estimate into an economical
interface.  Operational-space cancellation reduces the full joint
dynamics to a normalized translational error model.  The optimization variable is a corrective task
force, the persistent disturbance is augmented in the same force domain with
a centered cost that makes offset rejection unconditional on task weighting,
and the robot-specific realization enters the complete applied-torque
constraint rather than bounding the correction alone.  The result is a small
convex QP whose state-transition block can be precomputed.

Two controlled comparisons in \S\ref{sec:experiments} and
\S\ref{sec:tauactive} deliberately separate what each mechanism is
credited for, rather than crediting the receding-horizon MPC architecture
as a whole: the disturbance estimate, not receding-horizon prediction,
drives the reported offset-rejection improvement, established against a
non-predictive fixed-gain observer-impedance baseline using the identical
estimator; and the applied-torque row, not deeper prediction, keeps the
command inside a tightened budget, at no tracking cost down to the lowest
budget the task can meet, established with the row on, off, and relaxed
rather than by extending the
horizon. Receding-horizon MPC is the vehicle in which both mechanisms are
realized and evaluated together; neither mechanism's own credit depends
on the horizon being predictive as opposed to single-step, and the
experiments below test that directly rather than assuming it.

The contributions are therefore threefold:

\begin{enumerate}[leftmargin=*,itemsep=2pt,topsep=2pt]
\item a \textbf{double-integrator interaction-error representation}: a
task-error-centered operational-space predictor with a constant transition
matrix and configuration-scheduled input effectiveness, making physical
interaction a predictive quantity on a shared model rather than a property
re-derived per contact or configuration;
\item force-domain offset-free augmentation on that representation, aligning
the disturbance estimate, horizon prediction, and corrective task-force
variable; and
\item an applied-torque realization, on the same representation, that
constrains the sum of feedforward, translation, orientation, and null-space
torques, accompanied by conditional analytical statements, reproducible
simulation audits, and a paired estimator ablation on a physical UR5e
(\S\ref{sec:hardware}).
\end{enumerate}

The novelty is architectural and implementational: an explicit, reusable
boundary between a behavior-coordinate request and robot-specific torque
realization.  The evaluation
focuses on the causal C4--C5 estimator ablation and adds a
stiffness-and-damping-calibrated impedance, tuned anti-windup PI, and filtered force-cancellation comparison
under common joint-torque limits.  Boundary and passivity results are reported
only at the level actually exercised by the simulation.

\section{System Dynamics and Problem Formulation}

Consider an $n$-DOF torque-controlled serial manipulator; simulation
experiments use the Franka FR3 ($n=7$)~\cite{franka2023fr3} and hardware
experiments the Universal Robots UR5e ($n=6$)~\cite{universalrobots2024ur5e}.
The rigid-body equations of motion are:
\begin{equation}\label{eq:eom}
  M(q)\ddot{q}+C(q,\dot{q})\dot{q}+G(q)=\tau+J^\top(q)\mathcal{F}_h
\end{equation}
where $M(q)$ is the positive-definite inertia matrix;
$C(q,\dot{q})\dot{q}+G(q)$ is the Coriolis-plus-gravity vector
(available as \texttt{qfrc\_bias} in MuJoCo or via the robot model API);
$\tau$ is the commanded joint torque; and $\mathcal{F}_h\in\mathbb{R}^6$
is the human-applied wrench.
$J(q)=[J_v^\top(q);\;J_\omega^\top(q)]^\top\in\mathbb{R}^{6\times n}$
is the full geometric Jacobian, with $J_v\in\mathbb{R}^{3\times n}$ the
translational and $J_\omega\in\mathbb{R}^{3\times n}$ the rotational
sub-Jacobian~\cite{siciliano2009robotics}.
$\mathcal{F}_h$ is defined as positive when the human pushes
\emph{onto} the robot, so it enters~\eqref{eq:eom} with a positive
sign, adding to the generalized force that accelerates the joints.

The translational Jacobian satisfies $\dot{p}=J_v(q)\dot{q}$~%
\cite{siciliano2009robotics}.
The operational-space inertia~\cite{khatib1987unified} is:
\begin{equation}\label{eq:Lambda}
  \Lambda(q)=\bigl(J_vM^{-1}J_v^\top\bigr)^{-1}\in\mathbb{R}^{3\times3}
\end{equation}
$\Lambda(q)$ is symmetric positive-definite.
The dynamically-consistent
pseudoinverse~\cite{khatib1987unified,sentis2005synthesis} is
$\Jvb=M^{-1}J_v^\top\Lambda$, ensuring forces applied through
$J_v^\top$ produce no null-space acceleration.

After feedforward cancellation (Section~\ref{sec:layer1}), the human
force and joint friction appear as an aggregated Cartesian
\emph{acceleration} disturbance $d_{\text{acc}}(t)$ (units m/s$^2$ ---
kept notationally distinct throughout from its \emph{force}-equivalent
$d$, units N, introduced in Section~\ref{sec:layer2c} and used from Layer
2 onward), modelled as a random walk:
\begin{equation}\label{eq:distmodel}
  \dot{d}_{\text{acc}}=w(t),\quad w\sim\mathcal{N}(0,Q_d)
\end{equation}
capturing constant forces (sustained pushes, physiological tremor) and
slow-varying loads (unmodelled payloads, tool changes).
This decomposition is written in acceleration units; the estimator and QP of
Section~\ref{sec:layer2} operate on the equivalent \emph{force-form}
disturbance $d=-\Lambda(q)d_{\mathrm{acc}}$ (the quantity entering through $B_d(\rho_k)$ in
\eqref{eq:augstate}), the form realized in code. Only the interaction-wrench
contribution to $d$ is constant under a constant human force; the remaining
terms of $d_{\text{acc}}$ in \eqref{eq:dist} --- orientation coupling, the
Jacobian-derivative term, model error --- stay configuration- and
motion-dependent regardless, which is why $d$ is modelled as a random walk
rather than a known constant.

For the reference-holding task considered here, design a torque law $\tau$
such that for any constant bounded disturbance $d$ compatible with the
actuator constraints:
\begin{enumerate}
  \item $\|e(\infty)\|=0$, where $e=p_d-p$ (zero steady-state error);
  \item $\|e(t)\|$ is minimized over a finite prediction horizon;
  \item $q_i\in[q_{\min,i}+\epsilon,\;q_{\max,i}-\epsilon]$
        for all $i$ when the one-step safety filter is enforced.
\end{enumerate}

The optimization uses the following reduced representation.
\begin{definition}[Interaction-error predictor]\label{def:id}
The \emph{interaction-error predictor} is the augmented task-coordinate model
\[
  x_{k+1} = A\,x_k + B\,u_k + E\,d_k
\]
with $x_k$ the task error and error rate, $u_k$ the corrective task force, and
$d_k$ a force-equivalent persistent disturbance. It is a control-oriented
reduction of~\eqref{eq:eom}: a task-coordinate model used for prediction.
\end{definition}

Section~\ref{sec:layer2} constructs the predictor with constant
$A=A_d$, scheduled $B=B_d(\rho_k)$, and an augmented estimate of $d$. This
constant-$A$ property applies to the residual translational predictor only.
The complete torque realization remains robot dependent through
$M$, $C\dot q+G$, $J$, the full-pose null-space projector, joint limits, and
actuator constraints.

\section{Linear Interaction-Error Predictor and Finite-Horizon Realization}

\subsection{Layer 1 --- Feedforward Nonlinear Inversion}
\label{sec:layer1}

Classical Cartesian impedance control~\cite{hogan1985impedance} commands:
\begin{equation*}
  \tau = C\dot{q}+G+J_v^\top(\Lambda\ddot{p}_d+\mu+K_de+D_d\dot{e})
\end{equation*}
where $\mu=\Lambda J_vM^{-1}C\dot{q}-\Lambda\dot{J}_v\dot{q}$ collects
task-space Coriolis and centripetal terms~\cite{khatib1987unified},
embedding the full impedance law inside a single static feedback.
This treats $F_h$ as a disturbance to be passively rejected with no
predictive look-ahead and no mechanism to drive steady-state error to
zero.

The proposed Layer~1 explicitly separates these concerns.
Following the operational-space decomposition of~\cite{khatib1987unified},
the full torque command is partitioned into four channels via the
null-space projector $\Nbar=I-\bar{J}J$~\cite{sentis2005synthesis}, where
$\bar{J}=M^{-1}J^\top\Lambda_6$ is the dynamically-consistent pseudoinverse
of the full $6$-DOF pose Jacobian $J$ and
$\Lambda_6=(JM^{-1}J^\top)^{-1}\in\mathbb{R}^{6\times6}$ is the full-pose
operational-space inertia (both the translational MPC task and the
orientation channel are actively controlled, so the projector annihilates
the full pose task, not only its translational part):
\begin{equation}\label{eq:tau}
  \tau = \tau_{\text{ff}}+J_v^\top\Fmpc
        +J_\omega^\top F_{\text{orient}}+\Nbar^\top\tau_{\text{null}}
\end{equation}

The feedforward term cancels all known nonlinearities using the
computed-torque method~\cite{siciliano2009robotics}:
\begin{equation}\label{eq:tff}
  \tau_{\text{ff}} = \underbrace{C(q,\dot{q})\dot{q}+G(q)}_{\texttt{qfrc\_bias}}
                   + J_v^\top\Lambda(q)\,\ddot{p}_d
\end{equation}

Writing the Layer-2 task channel as $\tau_{\text{mpc}}\equiv J_v^\top\Fmpc$,
define the known non-MPC torque offset
\[
  \tau_{\text{base}}=
  \tau_{\text{ff}}+J_\omega^\top F_{\text{orient}}+\Nbar^\top\tau_{\text{null}}.
\]
The final applied command is the joint torque
$\tau = \tau_{\text{base}}+\tau_{\text{mpc}}$, which the hardware saturates at
the per-joint actuator limit $|\tau_i|\leq\tau_{\max,i}$; Layer~2 constrains
this applied receding-horizon torque, while the hardware interface keeps
clipping and rate limiting as final runtime guards.

\begin{proposition}\label{prop:doubleint}
After applying \eqref{eq:tff}, the residual error dynamics are represented by
the linear double integrator
\begin{equation}\label{eq:errordyn}
  \ddot{e} = -\Laminv\,\Fmpc + d_{\text{acc}}(t)
\end{equation}
with $e=p_d-p$; every term left over from the feedforward cancellation ---
not just the human force --- is collected into the \emph{acceleration}-form
disturbance $d_{\text{acc}}(t)$, defined below (not to be confused with its
force-equivalent $d=-\Lambda(q)d_{\text{acc}}$, Section~\ref{sec:layer2c},
which is the quantity actually estimated and cancelled from Layer 2
onward):
\begin{align}\label{eq:dist}
  d_{\text{acc}}(t) &= \underbrace{-\Laminv\,\bar{J}_v^\top J^\top\mathcal{F}_h}_{%
           \text{projected pHRI force}}
          \;\underbrace{-\,J_vM^{-1}J_\omega^\top F_{\text{orient}}}_{%
           \text{orientation coupling}}\notag\\
        &\quad\underbrace{-\,\dot{J}_v(q,\dot{q})\dot{q}}_{%
           \text{Jacobian derivative}}
          \;+\;\underbrace{\epsilon_m(t)}_{\text{model error}}
\end{align}
where $\bar{J}_v^\top=\Lambda J_vM^{-1}\in\mathbb{R}^{3\times n}$ and
$J^\top\in\mathbb{R}^{n\times 6}$ is the full geometric Jacobian transpose,
so $\bar{J}_v^\top J^\top\in\mathbb{R}^{3\times 6}$ acts on the $6$-D wrench
$\mathcal{F}_h$ (equivalently the term is $-J_vM^{-1}J^\top\mathcal{F}_h$).
\end{proposition}

The four terms: (i)~the \emph{projected pHRI force} — for a purely
translational $\mathcal{F}_h=[F_h^\top,0]^\top$ this reduces to
$-\Lambda^{-1}F_h$ in acceleration form (because $e=p_d-p$, any positive
acceleration of $p$ induced by the human force appears with a negative sign
in $\ddot e$);
(ii)~the \emph{orientation coupling} — residual translational
acceleration from the simple additive orientation PD. This mobility
cross-term is generally structural, not a numerical artifact, and is,
in its force-equivalent form, absorbed by $\dhat$;
(iii)~the \emph{Jacobian derivative} $-\dot{J}_v\dot{q}$, small at
typical speeds; (iv)~the \emph{model error} $\epsilon_m(t)$ from
imperfect gravity/Coriolis cancellation.

\begin{remark}[Exactness of the reduction]
Under \emph{ideal} cancellation the residual plant is exactly the double
integrator~\eqref{eq:errordyn}; in practice the terms
(ii)--(iv) are nonzero, and are
\emph{modelled} as the acceleration disturbance $d_{\text{acc}}(t)$, whose
force-equivalent $d=-\Lambda(q)d_{\text{acc}}$ the augmented Kalman state
$\dhat$ estimates and the MPC rejects (Theorem~\ref{thm:zss}). The reduction is
thus a \emph{cancellation-plus-disturbance} model. Only the residual translational predictor has constant $A_d$ and
scheduled $B_d(\rho_k)$; the applied-torque realization retains the full
robot-model dependence listed after Definition~\ref{def:id}.
\end{remark}

\subsection{Layer 2 --- Discrete LPV Model and QP}\label{sec:layer2}

With error state $x_e=[e^\top,\dot{e}^\top]^\top\in\mathbb{R}^6$ and
scheduling variable $\rho_k=q_k$, exact zero-order-hold (ZOH) discretization gives:
\begin{equation}\label{eq:lpv}
  x_{e,k+1} = \underbrace{\begin{bmatrix}I_3&\Delta t I_3\\0&I_3\end{bmatrix}}_{A_d\;\text{(constant)}}
              x_{e,k}
            + \underbrace{\begin{bmatrix}-\tfrac{\Delta t^2}{2}\Laminv\\-\Laminv\Delta t\end{bmatrix}}_{B_d(\rho_k)}
              \Fmpck{k}
\end{equation}

Equation~\eqref{eq:lpv}, together with the disturbance augmentation
\eqref{eq:augstate} below, instantiates Definition~\ref{def:id}: the state
$x_e$ is the task error, the input $\Fmpc$ is the corrective task force, and
$\hat d$ is the estimated force-equivalent disturbance. Identifying $A=A_d$
and $B=B_d(\rho_k)$ gives the model used by the QP; constraints are imposed
through a separate robot-specific torque map.

\textbf{Structural result.}
The continuous-time matrix $A_c=\begin{bmatrix}0&I_3\\0&0\end{bmatrix}$
is nilpotent ($A_c^2=0$), so the matrix exponential terminates exactly at
first order and $A_d=e^{A_c\Delta t}=I+A_c\Delta t$ is
configuration-independent and exact, with no discretization
approximation; the exact ZOH input integral $B_d(\rho_k)$ is likewise
closed-form. Each frozen pair $(A_d,B_d(\rho))$ is pointwise controllable,
and the disturbance state augmented in Sec.~\ref{sec:layer2c} is
detectable, at every configuration away from kinematic singularities
(proof: supplementary material, Sec.~S1, Remark~S1).

\textit{Implementation note (regularization).}
In code $\Lambda^{-1}(q)=J_vM^{-1}J_v^\top+\sigma I$ with $\sigma=10^{-6}$; this
Tikhonov term is negligible in the workspace interior and only mildly damps the
estimate near singularities, outside the operating region.

The receding-horizon QP with $N=10$ steps, $3N=30$ decision variables
$U=[\Fmpck{0};\ldots;\Fmpck{N-1}]$, is:
\begin{subequations}\label{eq:QP}
\begin{align}
  \min_U\;& \tfrac{1}{2}U^\top HU+h^\top U \label{eq:QP_cost}\\
  \text{s.t.}\;&
  -\tau_{\max}\leq
  \tau_{\text{base}}(i)+J_v^\top(q_i)\Fmpck{i}
  \leq\tau_{\max}, \quad i=0,\ldots,N_c\!-\!1, \label{eq:QP_torque}\\
  & -F_{\max}\leq\Fmpck{k}\leq F_{\max},
  \quad k=0,\ldots,N-1. \label{eq:QP_force}
\end{align}
\end{subequations}
where $F_{\max}$ is the task-space (Cartesian) corrective-force bound on the
QP's decision variable (distinct from the joint-torque limit $\tau_{\max}$),
and $\tau_{\text{base}}(i)=\tau_{\text{ff}}(i)+J_\omega^\top F_{\text{orient}}+\bar{N}^\top\tau_{\text{null}}$
is the known non-MPC torque offset (Sec.~\ref{sec:layer1}). Step~0 uses the
exact current-configuration values $\tau_{\text{ff}}(0),J_v(q_0)$; for
$i\geq1$, $\tau_{\text{ff}}(i)$ and $J_v(q_i)$ are evaluated at the
reference-scheduled configuration $q_i=q_d(t+i\Delta t)$ --- the best
available prediction of the future state, since the realized trajectory is
not yet known --- while orientation and null-space terms stay frozen at
their step-0 value. $N_c\in\{1,\ldots,N\}$ sets how many leading rows are
enforced; the reported experiments use $N_c=1$ (first-move only), and the
reference-scheduled extension to $N_c=N$ is implemented and available. $H=\Gamma^\top\bar{Q}\Gamma+\bar{R}$, stage cost
$Q=\text{blkdiag}(K_d,D_d)$, terminal cost $Q_f=\gamma Q$ ($\gamma=5$),
$\bar{Q}=\text{blkdiag}(Q,\ldots,Q,Q_f)$,
$\bar{R}=\text{blkdiag}(R,\ldots,R)$.
Let $d_N=\mathbf 1_N\otimes\dhat$; the linear term is
$h=\Gamma^\top\bar{Q}x_{\text{free}}+\bar R d_N$, corresponding to the
offset-free input-centered incremental effort penalty
$\|U+d_N\|_{\bar R}^2$ rather than absolute physical effort.
Equivalently, before eliminating the predicted state, the objective is
\begin{equation}\label{eq:centeredcost}
 J(U)=\tfrac12X^\top\bar QX
 +\tfrac12(U+d_N)^\top\bar R(U+d_N),
\end{equation}
whereas the physical force and torque constraints remain imposed on $U$.
The free response $x_{\text{free}}=\Phi x_e+D_{\text{bar}}\dhat$,
where $D_{\text{bar}}\in\mathbb{R}^{6N\times3}$ has $j$-th block
$\sum_{l=0}^{j-1}A_d^l B_d(\rho_k)$; equivalently,
$D_{\text{bar}}=\Gamma(\mathbf 1_N\otimes I_3)$.
Only $\Fmpck{0}$ is applied (receding-horizon principle).
\emph{Constraint interpretation (feasible final control).} The final
applied command is the joint torque
$\tau=\tau_{\text{base}}+J_v^\top\Fmpck{0}$, where
$\tau_{\text{base}}$ contains feedforward, orientation, and null-space terms,
which the hardware saturates at the per-joint actuator limit $\tau_{\max}$
(on the FR3, $\tau_{\max}=[87,87,87,87,12,12,12]$\,Nm). Constraining only the
correction, $\|\Fmpc\|_\infty\leq F_{\max}$, would \emph{not} keep $\tau$
within $\tau_{\max}$: it ignores the feedforward offset $\tau_{\text{ff}}$
(both $C\dot{q}+G$ and the inertial term $J_v^\top\Lambda\ddot{p}_d$, which on
the FR3 dominate the budget), and it bounds a Cartesian force whereas the map
$J_v^\top(q)$ to joint torque is configuration-dependent; any deficit would
be absorbed by a downstream torque clip, i.e.\ \emph{saturation}, discarding
the optimized command. We therefore constrain the first applied
receding-horizon torque directly, as in~\eqref{eq:QP_torque}. Since
$\tau_{\text{base}}(0)$ and $J_v(q_0)$ are known at solve time, the constraint
is affine in $\Fmpck{0}$ --- a convex per-joint polytope solved by OSQP at the
same cost class as a box. As only the first move is applied, the nominal
applied torque satisfies $|\tau_i|\leq\tau_{\max,i}$ whenever the QP is
feasible. Runtime clipping and rate limiting remain as final safety guards;
horizon-wide torque rows can be added but are not required for feasibility of
the applied receding-horizon action. If an over-aggressive reference makes
$\tau_{\text{base}}(0)$ alone exceed $\tau_{\max}$, the released
implementation returns a zero correction on infeasibility (correction-
authority ablation, supplementary material Sec.~S9), which is simple but
not minimum-violation. The $i=0$ ($N_c=1$) row can instead be softened with a
penalized slack $s\geq0$ so the QP stays feasible and returns a
minimum-violation-like command, analogous to the joint-limit CBF
relaxation of Sec.~\ref{sec:cbf}; \S\ref{sec:tauactive} implements and
stress-tests exactly this relaxation under a deliberately tightened
$\tau_{\max}$: it stays feasible where the hard row does not, but below the
torque the task needs it does not recover tracking.

The following conditional identification relates the unconstrained feedback
component to classical impedance under the assumptions stated in the theorem.

\begin{theorem}[Impedance Equivalence]\label{thm:equiv}
Fix a configuration and assume the input constraints are inactive and
the disturbance estimate is zero. Then the first MPC move is a linear
state feedback
$\Fmpc=K_e e+K_{\dot e}\dot e$. If the realized gain blocks are symmetric
positive definite, or more generally have symmetric parts that define passive
stiffness and damping, set
$K_{\text{eff}}=\operatorname{sym}(K_e)\succ0$ and
$D_{\text{eff}}=\operatorname{sym}(K_{\dot e})\succ0$; the conservative part
of the closed loop is the classical task-space impedance:
\begin{equation}\label{eq:impedance}
  \Lambda(q)\ddot{e}+D_{\text{eff}}\dot{e}+K_{\text{eff}}e = -F_h,
\end{equation}
where $F_h$ is the constant force-equivalent disturbance applied to the
end effector.
\end{theorem}
\begin{proof}
With inactive constraints and $\dhat=0$, \eqref{eq:QP} reduces to an
unconstrained strictly convex quadratic because
$H=\Gamma^\top\bar Q\Gamma+\bar R\succ0$ when $\bar R\succ0$. Its
unique minimizer is
\begin{equation*}
  U^\star=-H^{-1}\Gamma^\top\bar Q\Phi x_e .
\end{equation*}
Let $S=[I_3\;0\;\cdots\;0]$ select the first force block. The applied
force is therefore
\begin{equation*}
  \Fmpc=S U^\star =
  \begin{bmatrix}K_e&K_{\dot e}\end{bmatrix}
  \begin{bmatrix}e\\\dot e\end{bmatrix},
\end{equation*}
for fixed matrices $K_e,K_{\dot e}$ determined by
$(A_d,B_d,\bar Q,\bar R,Q_f,N)$. Under the stated symmetry/sign condition,
the symmetric parts define $K_{\text{eff}}$ and $D_{\text{eff}}$. The exact
linear closed loop is obtained with the full gains $K_e,K_{\dot e}$; only
their symmetric parts admit the conservative spring--damper interpretation.
The skew part of $K_{\dot e}$ is workless gyroscopic damping coupling, whereas
the skew part of $K_e$ is a circulatory position force and is not passive in
general. These skew terms vanish in the common isotropic/commuting case and
are otherwise treated as part of the stabilizing LQR feedback rather than as
classical impedance parameters. The residual plant
from Proposition~\ref{prop:doubleint} is
$\ddot e=-\Lambda^{-1}\Fmpc+d_{\text{acc}}$. For a translational
constant human force represented in force form by $F_h=-\Lambda
d_{\text{acc}}$, multiplication by $\Lambda$ gives
$\Lambda\ddot e=-\Fmpc-F_h$. Substituting the linear feedback yields
\eqref{eq:impedance}.
\end{proof}
Thus the unconstrained MPC realizes a configuration-dependent linear
impedance whenever the gain blocks have the passive stiffness/damping
sign structure; constraints and disturbance augmentation are the
mechanisms that extend this linear impedance behavior.

\begin{remark}[Optimality vs.\ passivity: motivation for the energy tank]
Because the predictive gain is the LQR/Riccati image of $(Q,R)$ rather than a
passive design, the skew part of $K_e$ (Theorem~\ref{thm:equiv}) is a
non-passive circulatory force. For the robot alone this is benign (the DARE
closed loop is asymptotically stable), but coupled to a human---itself a
dynamical system---a non-passive controller need not preserve stability of the
\emph{coupled} loop. This motivates the energy-tank layer of
\S\ref{sec:passivity}: rather than restricting $(Q,R)$ to passive gains and
sacrificing performance, we keep the optimal law and restore a sampled
energy-budget certificate on the scaled translational predictive-force
channel at runtime (Proposition~\ref{prop:tank}), bounding the positive work
that channel can deliver. This is a channel-level certificate, not a
full-port passivity proof (see the scope remark in
\S\ref{sec:passivity}), so it does not by itself guarantee coupled
stability against an arbitrary passive environment; it removes one
identified non-passive source while leaving feedforward, orientation, and
null-space contributions to the port outside its accounting.
\end{remark}

\begin{remark}[Prescribed vs.\ realized gains]\label{rem:gains}
The cost weights $Q=\mathrm{blkdiag}(K_d,D_d)$, $R$ are design
\emph{penalties}; the realized impedance
$(K_{\text{eff}},D_{\text{eff}})$ is their LQR image ---
a nonlinear (Riccati) map of $(Q,R)$ --- so in general
$(K_{\text{eff}},D_{\text{eff}})\neq(K_d,D_d)$. The weights \emph{shape}
the realized impedance: the high-bandwidth tracking reported below
corresponds to the stiff, well-damped $K_{\text{eff}}$ the chosen
$(Q,R)$ induce. If a specific $(K_d,D_d)$ must instead be matched
exactly, the impedance gain can be prescribed directly
($\Fmpc=K_de+D_d\dot{e}$), rendering the equivalence exact for those
gains at the cost of predictive look-ahead. We adopt the LQR form, for
which Theorem~\ref{thm:equiv} already establishes the analytical
connection between the unconstrained MPC and a realized classical
impedance.
\end{remark}

\begin{corollary}[Existence of a Gain-Scheduled Infinite-Horizon
Implementation]\label{cor:1khz}
The infinite-horizon MPC admits a pure state-feedback realization
$\Fmpck{k}=-K_\infty(\rho_k)x_{e,k}$. By the DARE detectability result
established in the supplementary material (Sec.~S1, Remark~S1), the
stabilizing DARE solution $P_\infty(\rho)$ exists for each fixed $\rho$
under the stated detectability assumptions, and $K_\infty(\rho_k)$ depends
on $\rho_k$ only through $\Laminv$ --- a symmetric $3\times3$ SPD matrix,
so the scheduling space is the 6-D manifold of its independent entries.
Since the stabilizing DARE solution is continuous in its coefficients,
$K_\infty(\cdot)$ is continuous over this compact domain.
\end{corollary}

\begin{remark}[Practical 1\,kHz implementation]
Continuity of $K_\infty$ over a compact 6-D domain makes an
offline-precomputed lookup attractive: the gain can be tabulated on a
grid of $\Laminv$ values and interpolated, reducing the online cost to a
single $3\times6$ matrix--vector product. We present this as a deployment
\emph{option}; grid density, interpolation error, and the resulting
closed-loop quality are implementation choices for the target platform.
The simulations run the
finite-horizon QP \eqref{eq:QP} at 100\,Hz with a 1\,kHz inner loop (500\,Hz on the
UR5e, \S\ref{sec:hardware}),
which also enforces the first applied torque constraint~\eqref{eq:QP_torque} and exploits
prediction look-ahead; a hardware demonstration of the 1\,kHz law is future work.
\end{remark}

\subsection{Kalman Disturbance Augmentation for Offset-Free Tracking}\label{sec:layer2c}

For the estimator we represent the disturbance in \emph{force form}: the
augmented state $d\in\mathbb{R}^3$ (units N) is the force-equivalent of the
task-acceleration disturbance of the residual dynamics
$\ddot{e}=-\Laminv\Fmpc+d_{\text{acc}}$, related by $d=-\Lambda(q)\,d_{\text{acc}}$.
This form enters the discrete model through the \emph{same} input matrix
$B_d(\rho_k)$ as the control $\Fmpc$, and---unlike the acceleration form
$d_{\text{acc}}=-\Lambda^{-1}(q)d$, which varies with configuration as
$\Lambda(q)$ changes---its dominant, interaction-wrench component is constant
for a constant human force, independent of $\Lambda(q)$; the remaining,
smaller contributions to $d$ stay configuration- and motion-dependent, and
the random-walk model \eqref{eq:augstate} accommodates that residual
variation rather than assuming $d(t)$ is exactly constant, consistent with
the constant-$d_\infty$ hypothesis Theorem~\ref{thm:zss} states for the
estimator's converged target. It is the quantity
the implementation stores as $\dhat$ and cancels via $\Fmpc=-\dhat$ at steady
state. Augment the state with this integrating disturbance, modelled as a
random walk $d_{k+1}=d_k+w_k$:
\begin{equation}\label{eq:augstate}
\begin{aligned}
  \begin{bmatrix}x_{e,k+1}\\d_{k+1}\end{bmatrix}
  &= \underbrace{\begin{bmatrix}A_d&B_d(\rho_k)\\0&I_3\end{bmatrix}}_{A_{\text{aug}}}
    \begin{bmatrix}x_{e,k}\\d_k\end{bmatrix}\\[2pt]
  &\quad + \begin{bmatrix}B_d(\rho_k)\\0\end{bmatrix}\Fmpck{k}
  + \begin{bmatrix}0\\w_k\end{bmatrix},
\end{aligned}
\end{equation}
with $w\sim\mathcal{N}(0,Q_w)$.
The integrating block (eigenvalues at $1$) delivers offset-free
rejection of constant disturbances by the internal-model principle,
while the process noise $w$ lets the filter track slowly-varying
disturbances. A steady-state Kalman filter observing $[e;\dot{e}]$
produces the estimate $\dhat$ of $d$ via the innovation and converges to
the true disturbance. Both $e$ and $\dot e$ are
measured directly (position from forward kinematics, velocity from
$J_v\dot q$), so $C=[I_6\;0]$ and the filter is $9$-dimensional --- the
$6$ error states and the $3$ integrating disturbance states, with no
extra velocity-estimation states. The estimate $\dhat$
enters the free response, pre-compensating the QP before the disturbance
accumulates in the error state.

\textit{Implementation note (process-noise scaling).}
For a disturbance-tracking bandwidth invariant to the QP rate, the random-walk
covariance $Q_w$ must scale with the sample time, $Q_w=Q_d\,\Delta t$ (the
exact discretization of the continuous random walk \eqref{eq:distmodel}); a
single raw discrete constant held fixed across rates implicitly retunes
estimator aggressiveness. The implementation scales $Q_w$ by the QP period
between the 100\,Hz and 500\,Hz configurations (Table~\ref{tab:params}), so
the continuous-time-equivalent process-noise density is held constant across
the rate comparison; C6/C7 use the correspondingly smaller $Q_w=2I_3$.

\begin{theorem}[Offset-Free Steady-State Tracking]\label{thm:zss}
Consider the augmented model \eqref{eq:augstate} at a fixed
configuration, with constant disturbance $d_k=d_\infty$. Assume:
(i) the augmented estimator is detectable and its estimation error is
asymptotically stable, so $\hat d_{k}\to d_\infty$; (ii) the unconstrained
centred-input finite-horizon regulator is asymptotically stable; and
(iii) the force and torque constraints are eventually inactive along the
considered trajectory. Then the
steady-state tracking error is zero:
$\lim_{k\to\infty}e(k)=0$. If $d_k\to d_\infty$ asymptotically, the
same conclusion holds provided the estimator remains asymptotically
tracking.
\end{theorem}
\begin{proof}
Let $\hat d_N=\mathbf 1_N\otimes\hat d$ and define the centred input sequence
$V=U+\hat d_N$. Since $D_{\text{bar}}=\Gamma(\mathbf 1_N\otimes I_3)$, the
identity
\[
  \Phi x_e+D_{\text{bar}}\hat d+\Gamma U=\Phi x_e+\Gamma V
\]
holds for any disturbance estimate $\hat d$, not only at steady state.
Crucially, the original effort term in~\eqref{eq:centeredcost} is
$\tfrac12(U+\hat d_N)^\top\bar R(U+\hat d_N)$, not
$\tfrac12U^\top\bar RU$. Therefore the change of variables gives exactly
\begin{equation*}
 J(V)=\tfrac12(\Phi x_e+\Gamma V)^\top\bar Q(\Phi x_e+\Gamma V)
 +\tfrac12 V^\top\bar R V .
\end{equation*}
Thus the nominal controller is exactly the stabilizing linear regulator for
the disturbance-free backbone in the variable $V$: $V^\star=-Kx_e$. The
applied first input is therefore $\Fmpc=v_0-\hat d$, and the plant evolves as
\begin{equation*}
  x_e^+=(A_d-B_dK_0)x_e+B_d(d_\infty-\hat d),
\end{equation*}
where $K_0$ is the first block row of $K$. By assumption (iii), this
unconstrained law applies after some finite time. By assumption (ii), its nominal
closed loop is asymptotically stable. By assumption (i), the additive input
$B_d(d_\infty-\hat d)$ vanishes, so the stable linear system converges to the
same equilibrium. Hence $x_{e,k}\to0$ and therefore $e_k\to0$.

The result covers constant or convergent disturbances with eventually
inactive constraints; arbitrary non-convergent time-varying forces and
steady states that require sustained actuator saturation are treated by a
separate invariant-set or anti-windup analysis.
\end{proof}

\begin{remark}[Frozen-configuration scope]
Theorem~\ref{thm:zss} is a \emph{frozen-configuration} offset-free result:
it fixes $\rho_k=q$ so that $B_d(\rho_k)$ and the regulator gain are
constant. A supplementary vertex-LMI calculation establishes a common fixed
stabilizer for the scheduled backbone over the sampled operating polytope;
that gain certifies the backbone rather than the deployed finite-horizon
MPC policy, whose per-configuration guarantee is Theorem~\ref{thm:zss}. The
moving-reference experiments in Section~\ref{sec:experiments} illustrate
that per-equilibrium property empirically, applying Theorem~\ref{thm:zss}
at each sustained-contact configuration the trajectory visits.
\end{remark}

\textbf{Scheduled-backbone certificate.}
Over the benchmark trajectory, 4700 samples of $\Lambda^{-1}(q)$ define a
64-vertex operating polytope. A feasible common-quadratic vertex LMI certifies
one fixed stabilizer with contraction factor $\varrho\le0.996$ for this sampled
scheduled backbone; the deployed receding-horizon policy's guarantee is the
per-configuration Theorem~\ref{thm:zss}. The full statement, matrix conventions, proof, containment audit, and
numerical certificate are therefore placed in the supplementary material
(Sec.~S2, Theorem~S1) and reproduced by
\texttt{lmi\_workspace\_stability\_probe.py}.

\subsection{Disturbance Prediction Validity over the Horizon}

A natural question for any offset-free MPC is how good the disturbance
prediction is \emph{over the prediction horizon}, not merely at the current
step. The QP propagates the augmented model
\eqref{eq:lpv},~\eqref{eq:augstate} under the random-walk prediction
$\dhat_{k+i\mid k}=\dhat_{k\mid k}$ --- a flat extrapolation that is exact only
for a strictly constant disturbance. For a time-varying aggregate disturbance
the prediction degrades with horizon depth $i$. We address this in three parts.

\textbf{Bound.} Assume the force-equivalent aggregate disturbance $d(t)$ is
Lipschitz continuous,
$\|\dot{d}(t)\|_2\le L_d$, and let $\|d_k-\dhat_{k\mid k}\|_2\le e_K$ be the
steady-state Kalman estimation error. Since the random-walk prediction is the
identity, the $i$-step prediction error obeys
\begin{equation}\label{eq:predbound}
  \|d_{k+i}-\dhat_{k+i\mid k}\|_2 \;\le\; e_K + L_d\,i\,\Delta t,
  \qquad i=0,\dots,N-1,
\end{equation}
where $L_d$ (units N/s) bounds the aggregate rate, including human force,
orientation/$\dot J_v\dot q$ coupling, model error, and configuration-dependent
force-equivalent terms. In S7's static-target experiment the applied sinusoid
is the dominant commanded variation, but $L_d$ is computed directly from the
reconstructed aggregate $d_k$. The quantity $e_K$ is the current Kalman error; the bound
follows by the triangle inequality on $d_{k+i}=d_k+\int\dot{d}\,d\tau$. The
first term is the \emph{current} estimation error (to be specified or
measured for the operating condition); the second is the \emph{extrapolation}
error, linear in horizon depth and unavoidable for any constant-disturbance
internal model.

\textbf{Why it stays benign.} Three properties keep
\eqref{eq:predbound} benign. (i)~\emph{Receding horizon:} only $\Fmpck{0}$ is
applied and the QP is re-solved every $\Delta t$ with a fresh
$\dhat_{k\mid k}$, so the loose late-horizon predictions never reach the plant
--- they only shape the near-term optimum. (ii)~\emph{Short horizon:} at
$N=10$, $\Delta t=10$\,ms (100\,ms; 20\,ms at 500\,Hz) the extrapolation term
is small in absolute units --- for a sustained push (the case of
Theorem~\ref{thm:zss}) $L_d\approx0$ and the flat-extrapolation term vanishes
(although current estimator error may remain). When the aggregate rate is
$L_d\lesssim5$\,N/s, the worst-case end-of-horizon error is
$L_d N\Delta t\lesssim0.5$\,N, a small fraction of the
available control authority. (iii)~\emph{Structured augmentation:} when the
human force has known temporal structure --- physiological tremor
(8--12\,Hz) or a slow guidance ramp --- the random-walk block can be replaced
by the corresponding internal model (a harmonic-oscillator state at the tremor
frequency, or a constant-plus-ramp double integrator) so that the component is
\emph{predicted} forward exactly rather than extrapolated flat. This grows the
estimator by a few states, leaves the constant-$A_d$ error block (and hence
the precomputed $\Phi$) untouched, and removes that component from $L_d$ in
\eqref{eq:predbound}.

\textbf{Metric.} The natural diagnostic is the \emph{$N$-step
disturbance-prediction RMS}, $\varepsilon_N=\mathrm{RMS}_k\,\|d_k-\dhat_{k\mid
k-N}\|$ --- the error of the estimate propagated $N$ steps before the
measurement --- monitored alongside the one-step value $\varepsilon_1$. For
the step-force benchmark (Section~\ref{sec:experiments}) the disturbance is
piecewise constant, so $\varepsilon_N\approx\varepsilon_1$ except across the
force-onset transient. In the supplementary time-varying-force experiment
(Sec.~S7, Table~S4), simulator-ground-truth $d_k$ is reconstructed from the
rigid-body acceleration identity and used to evaluate both $\varepsilon_1$
and $\varepsilon_N$ directly. Across the tested rates, the measured
$\varepsilon_N$ stays below the RMS counterpart of~\eqref{eq:predbound}, a
numerical check over the stated protocol.

\section{Orientation Stabilization and Null-Space Control}

\subsection{Orientation Stabilization}

Orientation is regulated by a PD law in operational
space~\cite{khatib1987unified}.
The axis-angle error $e_R=\tfrac12(RR_d^\top-R_dR^\top)^\vee\in\mathbb{R}^3$
is expressed in the world frame, like $\omega$~\cite{siciliano2009robotics},
and $R_d$ is the initial end-effector orientation:
\begin{equation}\label{eq:torient}
  \tau_{\text{orient}} = J_\omega^\top(-K_{\text{rot}}\,e_R - D_{\text{rot}}\,\omega)
\end{equation}
with $K_{\text{rot}}=20$\,Nm/rad, $D_{\text{rot}}=6$\,Nm$\cdot$s/rad,
following the operational-space impedance structure of~%
\cite{khatib1987unified,albu2007unified}.
This runs at 1\,kHz as a \emph{separate} control channel from the
translational QP in \eqref{eq:tau}, dynamically coupled to it:
the null-space projector $\bar N^\top$ there applies only to
$\tau_{\text{null}}$, and $J_\omega^\top F_{\text{orient}}$ couples
into translational acceleration through the mass matrix, as
\eqref{eq:dist}'s orientation-coupling term makes explicit. That residual
dynamic coupling is part of $d_{\text{acc}}(t)$~\eqref{eq:dist}, whose
force-equivalent $\dhat$ estimates and cancels.

\subsection{Null-Space Joint-Limit Safety and Invariance}\label{sec:cbf}

Because both the translational MPC channel and the orientation PD channel
are actively controlled, the primary task is the full 6-DOF pose; for the
FR3 ($n=7$) this leaves $n-6=1$ null-space DOF.
The null-space torque $\Nbar^\top\tau_{\text{null}}$ (where
$\Nbar=I-\bar{J}J$~\cite{khatib1987unified,sentis2005synthesis} is the
projector of the full pose Jacobian) does not affect end-effector position
or orientation~\cite{sentis2005synthesis}.

\textbf{Null-space inverse-barrier.}
Joint-limit regulation via null-space repulsive potentials is a standard
secondary objective~\cite{khatib1987unified,sadeghian2014task}.
Define the clearance (in rad) to the nearest limit:
\begin{equation*}
  \phi_i(q) = \min(q_i-q_{\min,i},\;q_{\max,i}-q_i)\ge 0
\end{equation*}
The inverse-barrier gradient~\cite{sadeghian2014task}:
\begin{equation}\label{eq:barrier}
  g_i(q) = \begin{cases}
    +k_b\!\left(\tfrac{1}{\phi_i}-\tfrac{1}{\delta_i}\right)
      & \phi_i<\delta_i,\;q_i\text{ near lower limit}\\
    -k_b\!\left(\tfrac{1}{\phi_i}-\tfrac{1}{\delta_i}\right)
      & \phi_i<\delta_i,\;q_i\text{ near upper limit}\\
    0 & \text{otherwise}
  \end{cases}
\end{equation}
with threshold $\delta_i=\eta(q_{\max,i}-q_{\min,i})$ ($\eta=0.10$),
$k_b=50$\,Nm.
Combined with a centering spring~\cite{sentis2005synthesis}:
\begin{equation}\label{eq:tnull}
  \tau_{\text{null}} = -k_{\text{null}}(q-q_0)+g(q)-d_{\text{null}}\dot{q}
\end{equation}
with $k_{\text{null}}=10$\,Nm/rad, $d_{\text{null}}=2$\,Nm$\cdot$s/rad,
toward the FR3 neutral pose
$q_0=[0,-0.785,0,-2.356,0,1.571,0.785]^\top$.

\textbf{Task-space workspace projection.}
When joints approach limits, the null-space alone cannot prevent
violations for task-constrained DOF~\cite{sadeghian2014task}.
We additionally project via the Jacobian
pseudo-inverse~\cite{khatib1987unified}:
\begin{equation}\label{eq:projection}
  \delta p = k_{\text{ws}}\,(J_vJ_v^\top+\epsilon_r I)^{-1}J_v\,g(q),
  \quad\|\delta p\|\leq p_{\max}
\end{equation}
with $k_{\text{ws}}=5\times10^{-4}$\,m/(Nm/rad), $p_{\max}=0.06$\,m,
Tikhonov regularization $\epsilon_r=10^{-8}$~\cite{siciliano2009robotics}.
The QP optimizes toward $p_{d,\text{eff}}=p_d+\delta p$.

\textbf{Conditional sampled-data joint-limit filter.}
The dual-barrier controller above provides practical joint-limit regulation. A conditional
forward-invariance certificate is obtained by adding a one-step
joint-limit control-barrier filter to the commanded torque. Let
\begin{equation*}
  h_i^-(q)=q_i-q_{\min,i}-\epsilon,\qquad
  h_i^+(q)=q_{\max,i}-\epsilon-q_i ,
\end{equation*}
and define the certified safe set
\begin{equation*}
  \mathcal C_\epsilon=\{q:\ h_i^-(q)\ge0,\ h_i^+(q)\ge0,\ i=1,\ldots,n\}.
\end{equation*}
At sampling time $k$, approximate the next joint position under a
constant torque over one servo period $T_s$ by
\begin{equation}\label{eq:qpred}
  q_{k+1|k}=q_k+T_s\dot q_k+
  \tfrac12T_s^2 M^{-1}(q_k)\bigl(\tau_k+\hat\tau_{\rm ext,k}-b_k\bigr),
\end{equation}
where $b_k=C(q_k,\dot q_k)\dot q_k+G(q_k)$,
$\hat\tau_{\rm ext,k}=J^\top(q_k)\hat{\mathcal F}_{h,k}$ is the measured or
estimated external generalized force, and $\tau_k$ is the final command after
the nominal MPC and null-space terms. If no reliable external-force estimate
is available, both the lower- and upper-limit inequalities in \eqref{eq:cbf}
should be robustly tightened using a bound on
$\|\tau_{\rm ext}-\hat\tau_{\rm ext}\|$ propagated through
$\tfrac12T_s^2M^{-1}(q_k)$. The safety filter selects the closest feasible
torque to the nominal command subject to
\begin{equation}\label{eq:cbf}
  h_i^\pm(q_{k+1|k})\ge(1-\alpha)h_i^\pm(q_k),
  \qquad i=1,\ldots,n,\quad \alpha\in(0,1].
\end{equation}
Since \eqref{eq:qpred} is affine in $\tau_k$, \eqref{eq:cbf} adds
linear inequalities to the final torque projection.

\begin{theorem}[Sampled-Data Joint-Limit Invariance]\label{thm:invariance}
Assume $q_0\in\mathcal C_\epsilon$, the model used in
\eqref{eq:qpred}, including $\hat\tau_{\rm ext,k}$ or a valid two-sided robust
tightening, matches the sampled plant at $k{+}1$ or encloses it within the
tightened lower/upper bounds, and the safety-filter constraints \eqref{eq:cbf}
are feasible and enforced at every sample. Then $\mathcal C_\epsilon$ is
forward invariant for the sampled joint positions: $q_k\in\mathcal C_\epsilon$
for all $k\ge0$.
\end{theorem}
\begin{proof}
Suppose $q_k\in\mathcal C_\epsilon$. Then
$h_i^\pm(q_k)\ge0$ for every joint. By \eqref{eq:cbf},
\[
  h_i^\pm(q_{k+1})=h_i^\pm(q_{k+1|k})
  \ge (1-\alpha)h_i^\pm(q_k)\ge0,
\]
where equality between $q_{k+1}$ and $q_{k+1|k}$ follows from the
one-step model-matching assumption. Hence $q_{k+1}\in\mathcal C_\epsilon$.
Induction from $q_0\in\mathcal C_\epsilon$ proves the claim.
\end{proof}

The theorem is deliberately conditional: it certifies invariance only
when the one-step constraints are feasible and enforced. In the
implementation the filter is the \emph{final} torque projection (after the
nominal MPC and null-space terms), so it takes priority over the soft
null-space barrier and workspace projection, which only \emph{bias} the
arm away from limits. When the constraints~\eqref{eq:cbf} cannot all be
met --- e.g., barrier and CBF pull in conflicting directions, or
$\tau_{\text{base}}$ alone is already limit-inducing --- they are relaxed
with a penalized slack $s\ge0$ ($\min\|\tau-\tau_{\text{nom}}\|^2+w\,s^2$
s.t.\ $h_i^\pm(q_{k+1|k})\ge(1-\alpha)h_i^\pm(q_k)-s$), returning the
minimum-violation command rather than reporting infeasibility; a logged
$s=0$ is exactly the per-sample certificate that the theorem held without
relaxation. The experiments
below report the empirical behavior of the dual-barrier joint-limit
regulation law;
hardware runs intended to claim certified joint-limit safety should log
the CBF residuals and the slack $s$. The certificate is sampled-data:
possible intersample motion is covered only to the extent that the margin
$\epsilon$ and any robust tightening dominate the within-sample excursion.

\begin{remark}[Robust tightening bound and ISSf]
Let $\delta\tau_k$ collect the unmodeled generalized force in the one-step
predictor (human-torque estimation error $J^\top(F_h-\hat F_h)$, finite-jerk
torque change, and inertial-model error), with $\|\delta\tau_k\|\le B_\tau$. The
induced one-step position error $\tfrac12\Delta t^2 M^{-1}\delta\tau_k$ is bounded
by $\Delta q_{\rm safe}=\tfrac12\Delta t^2\|M^{-1}\|_2 B_\tau$; tightening the
one-step constraints by $\Delta q_{\rm safe}$ makes
Theorem~\ref{thm:invariance} \emph{input-to-state safe}---the sampled positions
stay in $\mathcal C_\epsilon$ for any disturbance within $B_\tau$ provided
$\epsilon\ge\Delta q_{\rm safe}$. At $\Delta t=1$\,ms with
$\|M^{-1}\|_2\lesssim10$\,kg$^{-1}$ this is $\approx5\times10^{-6}B_\tau$ rad
($B_\tau$ in N$\cdot$m), well within the experimental $\epsilon$ margin.
\end{remark}

\subsection{Energy-Tank Passivity Layer for Co-Manipulation}\label{sec:passivity}

The offset-free predictive controller is designed for precision disturbance
rejection; intentional co-manipulation additionally needs a passivity
certificate at the human--robot port. We add a sampled energy tank on the
translational task channel. Let $F_k^0$ be the predictive task force before
passivity filtering, $v_k=\dot p_k$ the measured end-effector translational
velocity, and $E_k\in[E_{\min},E_{\max}]$ the tank energy. The layer scales
only the task force,
\begin{equation}\label{eq:tank_scale_force}
  F_k=s_kF_k^0,\qquad s_k\in[0,1],
\end{equation}
leaving feedforward, orientation, and null-space torques outside the tank
accounting. Define $P_k^0=\max\{(F_k^0)^\top v_k,0\}$ and optional human
recharge power
\begin{equation*}
  P_{h,k}=\gamma\max\{\hat F_{h,k}^\top v_k,0\},\qquad \gamma\in[0,1].
\end{equation*}
The scale is
\begin{equation}\label{eq:tank_scale}
s_k=
\begin{cases}
1, & P_k^0=0,\\[1mm]
\min\!\left(1,\dfrac{(E_k-E_{\min})/T_s+P_{h,k}}{P_k^0}\right), & P_k^0>0,
\end{cases}
\end{equation}
and the tank update is
\begin{equation}\label{eq:tank_update}
E_{k+1}=\operatorname{sat}_{[E_{\min},E_{\max}]}
\!\left(E_k+T_s(P_{h,k}-\max\{F_k^\top v_k,0\})\right).
\end{equation}

\begin{proposition}[Sampled Passivity of the Task Channel]\label{prop:tank}
If $E_0\ge E_{\min}$ and \eqref{eq:tank_scale}--\eqref{eq:tank_update}
are enforced at every servo sample, then $E_k\ge E_{\min}$ for all $k$ and
\begin{equation*}
\sum_{j=0}^{k-1}T_s\max\{F_j^\top v_j,0\}
\le E_0-E_{\min}+\sum_{j=0}^{k-1}T_sP_{h,j}.
\end{equation*}
\end{proposition}
\begin{proof}
By construction, \eqref{eq:tank_scale} gives
$T_s\max\{F_k^\top v_k,0\}\le E_k-E_{\min}+T_sP_{h,k}$. Substitution into
\eqref{eq:tank_update} gives $E_{k+1}\ge E_{\min}$ before saturation, and
the saturation preserves the lower bound. Summing the unsaturated balance
gives the stated dissipation inequality.
\end{proof}

\textbf{Scope.} Proposition~\ref{prop:tank} provides a channel-level
energy budget for the scaled translational predictive-force channel
$F_k$. The remaining channels are each \emph{locally} passive in
isolation --- the orientation channel is a passive PD
($K_{\text{rot}},D_{\text{rot}}>0$) with respect to its own
$(\tau_{\text{orient}},\omega)$ port, the null-space torque
redistributes energy within the redundant subspace, and the feedforward
gravity/Coriolis compensation is workless at steady state. These channels
can still deliver \emph{cross-coupling} power into translation through the
mass matrix: \eqref{eq:dist} retains an orientation-to-translation
coupling term $-J_vM^{-1}J_\omega^\top F_{\text{orient}}$, so
$\tau_{\text{orient}}$ can do translational work the tank accounting does
not see. A joint full-port certificate bounding the coupled channels
together is a separate extension.

In hardware logs, the certificate is the tuple
\texttt{passivity\_certified}, \texttt{passivity\_energy}, and
\texttt{passivity\_scale}. The default $\gamma=0$ gives a conservative
certificate; $\gamma>0$ allows a configured fraction of measured human input
work to recharge the tank.

\section{Real-Time Implementation}

The optimization dimension is independent of the number of joints; the
implementation is robot-specific, requiring $M$, $C\dot q+G$, the
full Jacobian and pose projector, joint limits, and actuator constraints.
Substituting those interfaces preserves the 30-variable Layer-2 QP.
The simulation reference targets the Franka FR3 via
FCI~\cite{franka2023fr3}; \S\ref{sec:hardware} applies the same QP to a UR5e.

\begin{table}[!t]
\caption{Computational Budget per Control Layer}
\label{tab:compute}
\renewcommand{\arraystretch}{1.0}
\centering
\setlength{\tabcolsep}{4pt}
\footnotesize
\begin{tabular}{@{}llcc@{}}
\toprule
\textbf{Layer} & \textbf{Computation} & \textbf{Rate} & \textbf{Budget}\\
\midrule
Feedforward & $\tau_{\text{ff}}$: \texttt{qfrc\_bias}$+J_v^\top\Lambda\ddot{p}_d$ & 1\,kHz & $<$0.1\,ms\\
QP          & OSQP, 30 vars, first torque row       & 100\,Hz  & $<$1\,ms\\
Kalman      & 9-state predict + update              & 100\,Hz  & $<$0.1\,ms\\
Orientation & $J_\omega^\top(-K_{\text{rot}}e_R-D_{\text{rot}}\omega)$ & 1\,kHz & $<$0.1\,ms\\
Null-space  & Barrier gradient + projection         & 1\,kHz   & $<$0.1\,ms\\
\bottomrule
\end{tabular}
\end{table}

\textbf{QP solver.}
We use OSQP~\cite{stellato2020osqp} with warm-starting, which in our
simulation implementation reduced cold-start latency from 5\,ms to under
0.5\,ms (single-threaded, commodity CPU). Table~\ref{tab:timing} reports
measured per-tick wall-clock time (Kalman step, $\Gamma$/$H$/$h$ rebuild,
warm-started OSQP solve, and torque assembly) for C5 and C7 over the
Benchmark~I protocol, on a single-threaded Python/OSQP implementation
(Apple M1 Max) --- a simulation-side timing characterization on
general-purpose hardware.

\begin{table}[!t]
\caption{Per-Tick Solve Time and Deadline Misses}
\label{tab:timing}
\renewcommand{\arraystretch}{1.0}
\centering
\footnotesize
\setlength{\tabcolsep}{3.5pt}
\begin{tabular}{@{}lrrrrrr@{}}
\toprule
\textbf{Rate} & \textbf{Mean} & \textbf{p50} & \textbf{p99} & \textbf{Max} & \textbf{Deadline} & \textbf{Misses}\\
 & \textbf{[ms]} & \textbf{[ms]} & \textbf{[ms]} & \textbf{[ms]} & \textbf{[ms]} &\\
\midrule
100\,Hz (C5) & 0.57 & 0.54 & 0.75 & 24.15$^\dagger$ & 10.0 & 1/2400\\
500\,Hz (C7) & 0.51 & 0.49 & 0.71 & 1.65 & 2.0 & 0/12000\\
\bottomrule
\end{tabular}
\end{table}
\noindent $^\dagger$One-time OSQP problem setup on the episode's first
tick; excluding it, the steady-state maximum is 1.04\,ms and both rates
have zero misses. Timing excludes the separate MuJoCo dynamics query.

Although $H=\Gamma^\top\bar{Q}\Gamma+\bar{R}$ varies with configuration
through $B_d(\rho_k)$, updating it requires only overwriting the
pre-allocated non-zero elements of the upper-triangular CSC sparse
matrix — a $O(N^2)$ coefficient update completed in under 0.05\,ms via
\texttt{osqp\_update\_P}, preserving OSQP's factorization-reuse
structure. The linear term
$h=\Gamma^\top\bar{Q}x_{\text{free}}+\bar R d_N$ is updated similarly
via the OSQP linear-cost update; the $+\bar R d_N$ term is required by
the offset-free input-centering argument in Theorem~\ref{thm:zss}.

\textbf{ZOH policy.}
Between QP solves, $\Fmpc$ is held constant.
Layer~1 ($\tau_{\text{ff}}$) is recomputed at 1\,kHz from fresh
$(q,\dot{q})$, continuously cancelling gravity and Coriolis.

\textbf{Simulation.}
All simulations use the FR3 model from MuJoCo
Menagerie~\cite{mujoco_menagerie} with the MuJoCo physics
engine~\cite{todorov2012mujoco} at 1\,kHz. The hardware-path MuJoCo
verification suite runs, in simulation, the FR3 interface code written for deployment,
including torque-rate limiting, the one-step CBF projection, and the
energy-tank passivity filter; direct benchmark scripts are used only for
algorithm comparisons and figures.

Every numeric constant used across the experiments of
Section~\ref{sec:experiments} is consolidated in Table~S1 of the
supplementary material.

\section{Experiments}\label{sec:experiments}

We evaluate the controller in simulation on two dynamic benchmarks and two
stress tests, followed by one focused fairness study; \S\ref{sec:hardware}
reports the hardware evaluation. \textbf{DI-MPC} denotes the proposed double-integrator MPC
realization, and \textbf{MPVIC} is the predictive variable-impedance
comparator, in the class of~\cite{liu2025model,roveda2019optimal}. MPVIC
rolls out the same horizon and uses the same estimate $\hat d$ as the
proposed controller, but only to \emph{schedule} the apparent stiffness
$K^\star\in\{200,\dots,3000\}$\,N/m rather than to cancel the force ---
isolating variable-impedance adaptation from offset-free augmentation.
Controllers in the original benchmark use $K_d=300$\,N/m nominal, critically damped, with the
inner feedforward loop running at 1\,kHz for every variant. Tables
report every controller; plots show four representative curves
(D1--D3, D7) for readability. Each circular-trajectory episode starts at rest
on the reference, at $p_d(0)$ with the set-point orientation.

\begin{table}[!t]
\caption{Benchmarked Controllers}
\label{tab:controllers}
\renewcommand{\arraystretch}{1.0}
\footnotesize
\begin{tabular}{lll}
\toprule
\textbf{ID} & \textbf{Controller} & \textbf{Disturbance estimator}\\
\midrule
C1 & Classical Impedance   & ---\\
C2 & Admittance ($K_a=100$\,N/m) & Virtual spring\\
C3 & PI Impedance ($K_{\text{int}}=80$\,N/(m$\cdot$s)) & Integral\\
MPVIC & Var.-Imp.\ MPC, 100\,Hz & Kalman (sched.)\\
C4 & DI-MPC, 100\,Hz    & None\\
C5 & DI-MPC + Kalman, 100\,Hz & Augmented Kalman\\
C6 & DI-MPC, 500\,Hz    & None\\
C7 & DI-MPC + Kalman, 500\,Hz & Augmented Kalman\\
\bottomrule
\end{tabular}
\end{table}

\textbf{Benchmark I --- circular trajectory under step force.}
The end-effector tracks a 12\,cm radius circle in the XZ sagittal plane
($\omega=2\pi/8$\,rad/s) while a 15\,N step force $F_h=[0,0,-15]$\,N is
applied for 3\,s of every 8\,s cycle (3 cycles, 24\,s total); metrics
are averaged over the three force events.

\begin{table}[!t]
\caption{Benchmark I --- Circular Trajectory, 15\,N Step Force}
\label{tab:bm1}
\renewcommand{\arraystretch}{1.0}
\footnotesize
\begin{tabular}{lcccc}
\toprule
\textbf{Controller} & \textbf{RMS} & \textbf{RMS} & \textbf{Peak} & \textbf{SS}\\
 & \textbf{total} & \textbf{contact} & \textbf{defl.} & \textbf{err.}\\
 & \textbf{[mm]} & \textbf{[mm]} & \textbf{[mm]} & \textbf{[mm]}\\
\midrule
C1 -- Impedance     & 31.4 & 41.1 & 52.7 & 44.7\\
C2 -- Admittance    & 112.7 & 174.8 & 212.4 & 185.1\\
C3 -- PI Impedance  & 23.7 & 29.5 & 49.6 & 23.9\\
MPVIC -- Var.-Imp. MPC & 3.6 & 4.6 & 7.9 & 4.9\\
C4 -- MPC 100\,Hz   & 1.5 & 2.1 & 2.9 & 2.7\\
C5 -- MPC+K 100\,Hz & 0.50 & 0.54 & 2.5 & 0.033\\
Observer -- fixed gain+K & 0.51 & 0.54 & 2.5 & 0.035\\
C6 -- MPC 500\,Hz   & 0.58 & 0.81 & \textbf{1.1} & 1.1\\
C7 -- MPC+K 500\,Hz & \textbf{0.31} & \textbf{0.33} & 1.49 & \textbf{0.031}\\
\bottomrule
\end{tabular}
\end{table}

\begin{figure}[!t]
  \centering
  \includegraphics[width=0.8\columnwidth]{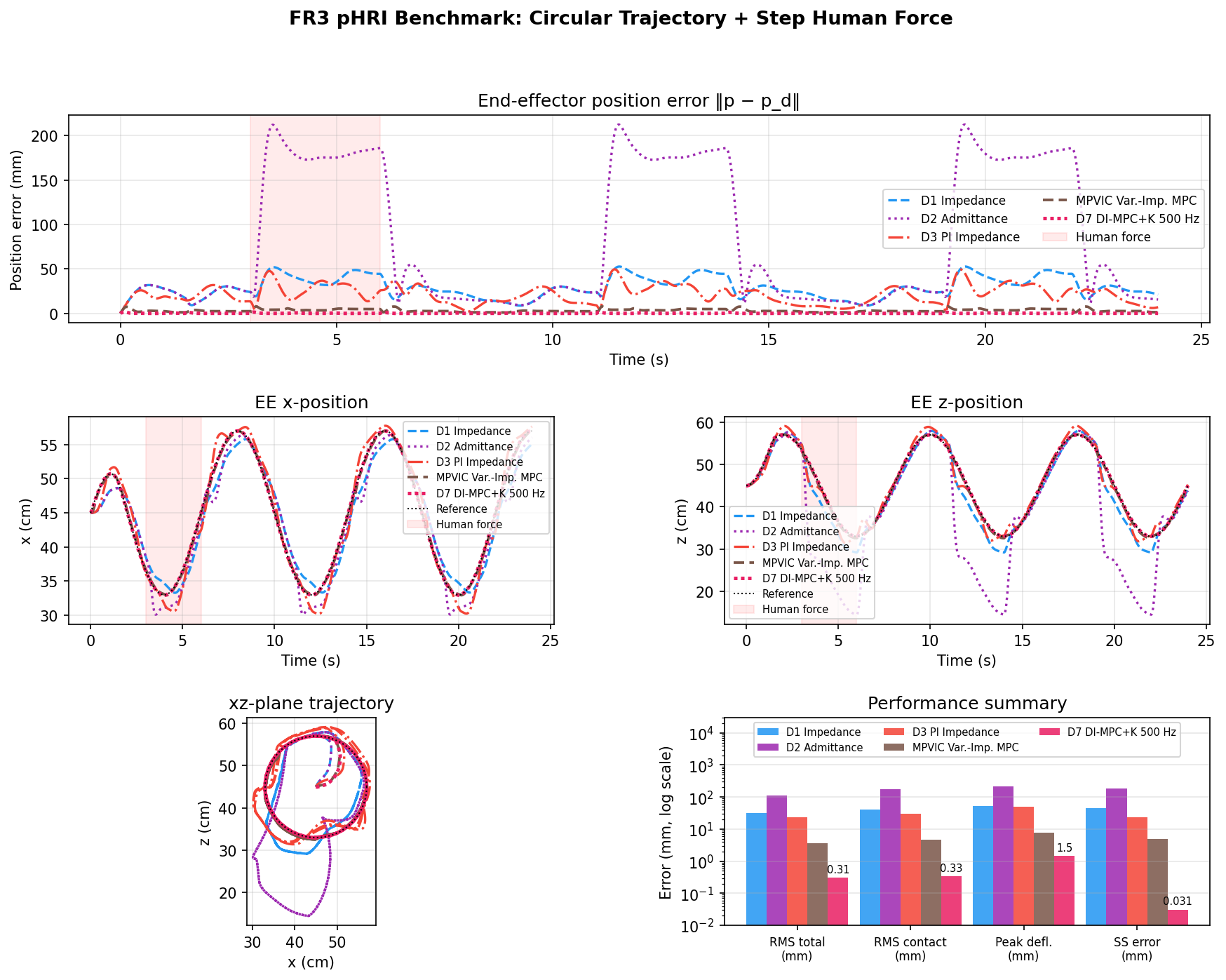}
  \caption{Benchmark I --- Circular trajectory under 15\,N step pHRI.
           For visibility the plot shows D1/D2/D3/D7 (corresponding to
           C1/C2/C3/C7) plus the predictive variable-impedance MPVIC
           baseline; Table~\ref{tab:bm1} reports the full ablation. The
           final D7 controller achieves near-zero steady-state error and
           sub-millimeter peak deflection, while MPVIC---though predictive---retains
           a visible steady-state offset.}
  \label{fig:bm1}
\end{figure}

The original comparison isolates the estimator and update-rate mechanisms;
the gain-confounded C1--C4 contrast is resolved separately below.

\begin{itemize}[leftmargin=*,itemsep=1pt,topsep=2pt]
\item \textbf{Kalman augmentation removes steady-state error (C4 vs.\
C5).} With QP rate and horizon held fixed, folding $\hat d$ into the
free response drives steady-state error to 0.033\,mm --- down from
2.7\,mm over C4 and 44.7\,mm over C1. Only the offset-free
disturbance-augmentation mechanism differs between C4 and C5; horizon,
weights, update rate, and torque interface are identical.
\item \textbf{The estimator, not horizon prediction, drives the
improvement (Observer vs.\ C5).} To isolate the estimator's credit from
the receding-horizon optimization's, the \emph{Observer} row tracks
$F=K_ee+K_{\dot e}\dot e-\hat d$ with $[K_e,K_{\dot e}]$ fixed to C5's own
realized unconstrained first-move gain (Theorem~\ref{thm:equiv}), evaluated
once, so no horizon is re-optimized online; it uses the identical Kalman
estimator. It matches C5 closely (0.54\,mm contact RMS and 2.5\,mm peak
for both) and nearly matches its steady-state error (0.035 vs.\
0.033\,mm), confirming that the estimator accounts for most of the
reported improvement under nominal actuation headroom.
\S\ref{sec:tauactive} isolates the effect of embedding and relaxing the
applied-torque row under active actuation limits, by tightening the
torque budget until it binds.
\item \textbf{Nominal impedance versus MPC is gain-confounded (C1 vs.\ C4).}
The 300\,N/m C1 setting is intentionally compliant, distinct from C4's
realized closed-loop gain, so their ratio reflects tuning rather than
architecture; Table~\ref{tab:fair} supplies a calibrated comparison.
\item \textbf{Offset-free cancellation outperforms stiffness adaptation (MPVIC vs.\ C5).}
MPVIC uses the \emph{same} $\hat d$ as C5 but only to schedule
stiffness, cutting contact RMS to 4.6\,mm --- better than every
reactive baseline --- yet it retains a 4.9\,mm steady-state error,
since a finite apparent stiffness can only bound deflection at
$e_\infty=-K^{\star-1}\hat d\neq0$. Offset-free cancellation
(Theorem~\ref{thm:zss}) removes this residual entirely ($\sim$114$\times$).
MPVIC is our own discrete-stiffness scheduler in the family of
\cite{liu2025model,roveda2019optimal}, with a fixed stiffness set and
force penalty.
\item \textbf{Rate and estimation interact (C4--C7).}
Raising the QP rate from 100 to 500\,Hz mainly shrinks the
zero-order-hold transient; adding the Kalman estimator mainly removes
the steady-state error. C7 combines both and gives the lowest
contact-window RMS and, of the exactly reported values, steady-state
error of any controller. Under the rate-proportional $Q_w$ scaling of
Table~\ref{tab:params}, however, C7 no longer has the lowest peak
deflection: C6's 1.1\,mm undercuts C7's 1.49\,mm, since at 500\,Hz the
faster rate alone already damps the first-contact transient and the
estimator's additional corrective action adds excess control effort
during that same window. C7 also has the lowest total RMS (0.31\,mm against
0.50\,mm for C5), since the faster rate reduces the free-motion error as
well (0.27 against 0.43\,mm outside contact).
\item \textbf{Admittance yields by design (C2).} Its 174.8\,mm
contact-window RMS is close to the equilibrium prediction
$15/100=150$\,mm, the behavior expected of intentional compliance.
\end{itemize}

\textbf{Focused fair comparison under a common actuator budget.}
To remove the stiffness confound, a separate one-cycle calibration first
estimates C4's equilibrium stiffness as
$\|F_h\|/\|e_{\rm ss}\|=5496$\,N/m. A $3\times3$ grid over stiffness scale
$(0.8,1.0,1.2)$ and damping ratio $(0.7,1.0,1.3)$ then minimizes the sum of
contact-onset trajectory RMSE and steady-state mismatch to C4; it selects
$K=5496$\,N/m and $\zeta=1.3$. A torque-aware PI baseline is independently
tuned over nine $(K_i,I_{\max})$ pairs. We also include first-order filtered
measured-force cancellation ($50$\,ms time constant) on the
\emph{stiffness-and-damping-calibrated} impedance backbone. It receives the exact simulated
wrench and is therefore a favorable sensor-based comparator. This is a deliberately simple disturbance-rejection
baseline motivated by force-estimation and ADRC approaches
\cite{jordana2024force,wu2025ensuring}. Final metrics use three evaluation cycles of the same periodic
trajectory. Every method uses the same
1\,kHz realization and FR3 torque limits
$[87,87,87,87,12,12,12]$\,Nm.

\begin{table}[!t]
\caption{Focused Fair Comparison, 15\,N Repeated Step}
\label{tab:fair}
\centering
\footnotesize
\setlength{\tabcolsep}{2.3pt}
\begin{tabular}{@{}lrrrr@{}}
\toprule
\textbf{Controller} & \textbf{Contact} & \textbf{SS} & \textbf{Free} & \textbf{Peak $P^+$}\\
 & \textbf{RMS [mm]} & \textbf{[mm]} & \textbf{RMS [mm]} & \textbf{[W]}\\
\midrule
Nominal impedance, 300\,N/m & 41.07 & 44.66 & 23.71 & 8.5\\
Calibrated impedance & 2.34 & 2.63 & 1.30 & 7.1\\
Tuned anti-windup PI & 27.80 & 20.68 & 15.16 & 12.6\\
Matched imp.+measured-force FF & 1.28 & 1.36 & 1.29 & 7.2\\
C4: DI-MPC 100\,Hz & 2.08 & 2.73 & 0.98 & 7.8\\
C5: DI-MPC+estimate 100\,Hz & \textbf{0.54} & \textbf{0.033} & \textbf{0.44} & 7.9\\
\bottomrule
\end{tabular}
\end{table}

\begin{figure}[!t]
  \centering
  \includegraphics[width=\columnwidth]{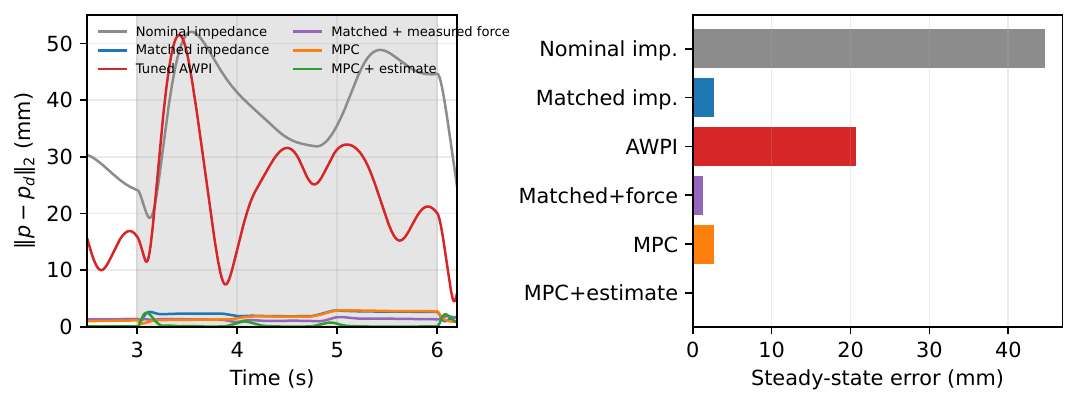}
  \caption{Held-out three-cycle fair comparison. Calibration largely closes
  the nominal-impedance/C4 gap, whereas the C4--C5 pair isolates the effect of
  force-domain disturbance augmentation. Direct measured-force feedforward
  also removes most load-axis bias.}
  \label{fig:fair}
\end{figure}

The calibrated impedance and C4 have similar steady-state errors
(2.63 and 2.73\,mm), so the original C1--C4 ratio reflects tuning rather
than a prediction capability. Their actuator use is also similar: no controller in
Table~\ref{tab:fair} reaches a joint limit, peak positive joint power is
7.1--7.9\,W for the calibrated and MPC controllers, and 7.1\,W during contact for
calibrated impedance, C4 and C5. No torque-rate, CBF, or tank filter is used in this focused
script. Both MPC variants complete all 2400 QP solves with zero failures.
Most importantly, C4 and C5 have
identical horizon, weights, and torque interface; adding the disturbance state
reduces steady-state error by $82\times$ at essentially the same peak positive
joint power (7.8 and 7.9\,W). The
matched measured-force baseline reduces total steady-state error to 1.36\,mm
and force-axis error to 0.42\,mm. Thus current-step force cancellation also
rejects the constant load, as expected; its remaining error is moving-reference
tracking in the other task directions. C5 reaches 0.033\,mm total error
without direct force measurement. The PI baseline improves on nominal
impedance; C5 is lower still in this experiment. On this benchmark, then, the
estimator's advantage over the calibrated baselines is offset removal without a
force sensor, not lower actuator use.

\textbf{Empirical joint-limit avoidance by the dual soft barrier
(boundary test).}
A 20\,cm radius circle pushes the arm toward its workspace boundary
(Table~\ref{tab:safety}, Fig.~\ref{fig:safety}), exercising the soft
null-space barrier and workspace projection; the final CBF filter
(Theorem~\ref{thm:invariance}) is exercised in the hardware-path
verification suite instead.

\begin{itemize}[leftmargin=*,itemsep=1pt,topsep=2pt]
\item Classical impedance violates the 5\% margin; the dual-barrier MPC
avoids it, holding 0.085 and 0.084 margins at 100 and 30\,Hz ---
joint-limit avoidance is insensitive to QP rate.
\item The elevated MPC RMSE (10.3\,mm vs.\ 0.3\,mm at $R=12$\,cm)
reflects the workspace projection deliberately trading tracking
accuracy for clearance near the boundary.
\end{itemize}

\begin{table}[!t]
\caption{Boundary Test, $R=20$\,cm Circle}
\label{tab:safety}
\renewcommand{\arraystretch}{1.0}
\centering
\footnotesize
\setlength{\tabcolsep}{2.5pt}
\footnotesize
\begin{tabular}{@{}lccc@{}}
\toprule
\textbf{Controller} & \textbf{RMSE [mm]} & \textbf{Min margin} & \textbf{Result}\\
\midrule
IMP         & 27.41 & 0.037 & \textsf{LIMIT VIOLATION}\\
MPC (C5, 100\,Hz QP) & 10.33 & 0.085 & \textsf{AVOIDED}\\
MPC (C5, 30\,Hz QP)  & 10.71 & 0.084 & \textsf{AVOIDED}\\
\bottomrule
\end{tabular}
\end{table}

\begin{figure}[!t]
  \centering
  \includegraphics[width=0.8\columnwidth]{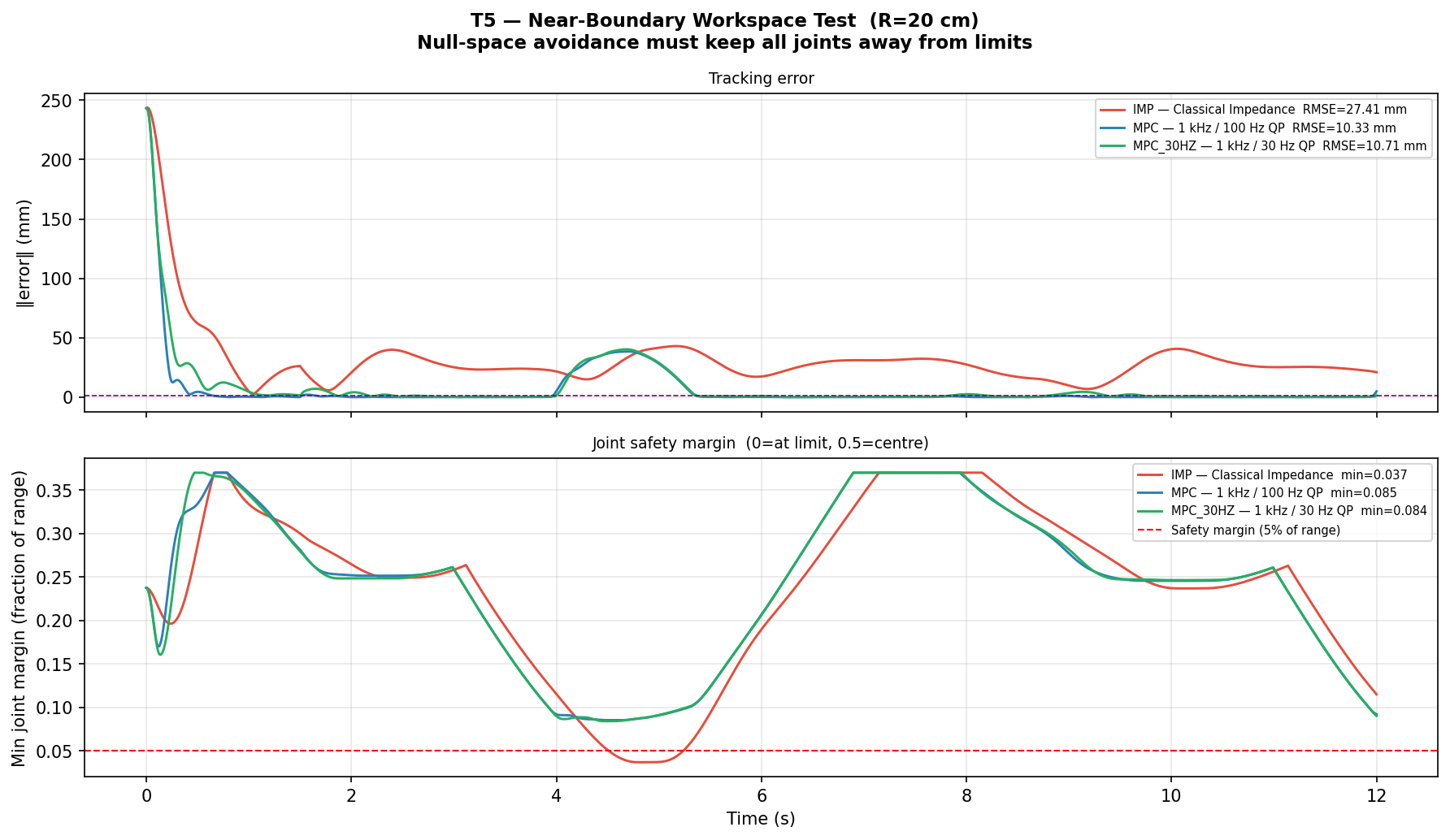}
  \caption{Boundary test ($R=20$\,cm).
           Classical impedance violates the 5\% joint margin;
           the dual-barrier MPC maintains a fractional joint margin
           $\geq0.084$ throughout.
           Elevated RMSE (10.3\,mm) reflects the workspace projection
           intentionally offsetting $p_d$ by up to 6\,cm to maintain
           joint-limit clearance.}
  \label{fig:safety}
\end{figure}

\textbf{Robustness to position-measurement noise and model mismatch.}
The architecture depends on a Kalman estimator, sensitive to
measurement noise, and on Layer-1 cancellation, sensitive to model
error, so we stress-test each independently: the plant is always the
true model, and only the controller's inputs (added EE-position noise)
or its model (the feedforward inertia $\Lambda$, scaled up by 0 to
30\% relative to the true plant) are perturbed. These are the two
dominant sensitivities of the architecture: position-measurement noise
into the Kalman estimator, and inertia mismatch into Layer-1
cancellation.

\begin{table}[!t]
\caption{Robustness of MPC+Kalman (100\,Hz QP)}
\label{tab:robust}
\renewcommand{\arraystretch}{1.0}
\centering
\footnotesize
\begin{tabular}{cccc}
\toprule
\multicolumn{4}{l}{\emph{(a) EE position noise --- 10\,N step, SS error
(5 seeds)}}\\
\midrule
noise $1\sigma$ & 0\,mm & 1\,mm & 2\,mm / 5\,mm\\
SS err (mm) & 0.08 & $0.20\pm.03$ & $0.37\pm.04$ / $0.89\pm.06$\\
\midrule
\multicolumn{4}{l}{\emph{(b) Inertia mismatch --- free-circle RMSE (mm)}}\\
\midrule
inertia error & 0\% & +10\% / +20\% & +30\%\\
Kalman on  & 0.31 & 0.30 / 0.30 & 0.29\\
Kalman off & 1.07 & 0.97 / 0.89 & 0.82\\
\bottomrule
\end{tabular}
\end{table}

\begin{itemize}[leftmargin=*,itemsep=1pt,topsep=2pt]
\item Under a 10\,N step (last 0.5\,s, 5 seeds), even 5\,mm ($1\sigma$)
position noise leaves steady-state error below 1\,mm and raises the transient
peak only from 1.4 to 2.2\,mm.
\item Under free-space XZ-circle tracking with a deliberately wrong
feedforward inertia (bottom table), Kalman-augmented RMSE stays
nearly unchanged through 30\% inertia error ($0.31\to0.29$\,mm), whereas
disabling Kalman gives 0.82--1.07\,mm (2.8--3.5$\times$ worse), isolating the
disturbance state's contribution.
\end{itemize}

\textbf{Further experiments and parameters.} The supplementary material
contains all parameters (Table~S1), a ground-truth horizon-prediction check
(S4), a force magnitude/shape sweep (S5), and the C8 correction-authority
ablation (S6).

\section{Hardware Evaluation on a UR5e}\label{sec:hardware}

\textbf{Setup.} The controllers run on a PC connected to a UR5e
(PolyScope 5.25) by wired RTDE at a nominal 500\,Hz. The QP and Kalman filter
keep the 100\,Hz rate and 30-variable structure of the simulation; the
remaining layers run at 500\,Hz, the robot's torque interface supplies gravity
compensation, and the robot program and PC stop the arm on a lost update, on a
joint speed above a guard (0.30\,rad/s in the later runs), or on a sustained
position error above 30\,mm. The torque budget
is 30\% of the joint limits, i.e.\ $[45,45,45,8.4,8.4,8.4]$\,Nm, with a
500\,Nm/s rate limit, and the measured actuation delay is 5.6--6.6\,ms over
three days. C4 and
C5 share all MPC gains with orientation stiffness $K_{\text{rot}}=60$.
C1 is the impedance response-matched to C4 in simulation
(4668\,N/m, $\zeta=1.3$). PID6000 adds to that impedance an anti-windup integral at
$K_i=6000$, chosen in simulation and qualified on this robot; simulation
limits this integral through friction near $K_i=9000$, a gain not run on the robot. The wrist
force/torque sensor scores load events only; no controller reads it.

\textbf{Protocol.} A nominal 1.5\,kg mass is hung by hand during a hold at a
fixed pose and lifted about 15\,s later, or 35\,s where stated. Load onset and removal are detected
from the force/torque signal, and the endpoint is the mean position error over
1\,s ending 0.2\,s before the lift. Runs invalidated by load handling are listed
in the run manifest, excluded, and repeated. The estimator ablation
uses a free-hanging mass and the Kalman process noise of the simulation. Later
runs use a guided attachment and five times that process noise, adopted before
those runs because the original value rings after steps against joint friction.

\textbf{Estimator ablation.} Table~\ref{tab:hw} and
Fig.~\ref{fig:hw} report five alternated C4--C5 pairs from one
session. C4 keeps an endpoint error of $1.505\pm0.108$\,mm and C5 reaches
$0.0231\pm0.0007$\,mm (mean $\pm$ SD), with a per-pair ratio of 61 to 77 and a
median of 63. In pair~1 the C4 load was hung about 7\,s before the cue and stayed
on for about 25\,s rather than 15\,s, at the full weight and steadily; without
that pair the values are $1.458\pm0.023$ and $0.0233\pm0.0005$\,mm, with a ratio
of 61 to 65. C5 is therefore offset-free on the robot: its error under load is the
noise level of the position measurement (joint encoders through forward
kinematics), below its own 0.028--0.059\,mm before the load. The ratio mainly
reflects that noise floor, so the load-induced change along
gravity is the more meaningful contrast: $-1.428\pm0.092$\,mm against
$-0.004\pm0.008$\,mm, zero within its spread. C4 sags less than simulation predicted (2.45--3.15\,mm)
probably because static friction carries part of the load. C5's estimate steps
$1.27\pm0.07$ times the measured load, consistent with an estimate that lumps
contact with friction rather than measuring force. A
robot-side realization, with the same unconstrained feedback and estimator
executing in URScript, repeats the contrast on separate days:
$1.367\pm0.090$ against $0.038\pm0.026$\,mm over three pairs, then 1.96 and
1.73 against 0.022\,mm over two further pairs. It did not improve on the PC
controller by a rule fixed in advance (better in 2 of 4 pairs, median ratio
0.99), so the results below use the PC path.

\begin{table}[!t]
\caption{UR5e Held 1.5\,kg Load (Range Over Runs)}
\label{tab:hw}
\renewcommand{\arraystretch}{1.0}
\centering
\footnotesize
\setlength{\tabcolsep}{2.5pt}
\begin{tabular}{@{}lcccc@{}}
\toprule
\textbf{Controller} & \textbf{Runs} & \textbf{Endpoint} & \textbf{Peak} & \textbf{Below 0.1\,mm}\\
 & & \textbf{[mm]} & \textbf{[mm]} & \textbf{from onset [s]}\\
\midrule
C4 & 5 & 1.425--1.696 & 1.48--1.74 & never\\
C5, ablation session & 5 & 0.022--0.024 & 0.34--0.99 & 1.7--2.9\\
\midrule
C1 (4668\,N/m) & 3 & 1.458--1.743 & 1.51--1.80 & never\\
PID6000, 35\,s load & 3 & 0.021--0.042 & 0.40--1.03 & 6.9--25.0\\
C5, 35\,s load & 3 & 0.022--0.024 & 0.20--0.40 & 1.3--10.6\\
C5, 15\,s load & 4 & 0.021--0.023 & 0.24--0.30 & 1.2--1.6\\
\bottomrule
\end{tabular}
\end{table}

\begin{figure*}[!t]
  \centering
  \includegraphics[width=\textwidth]{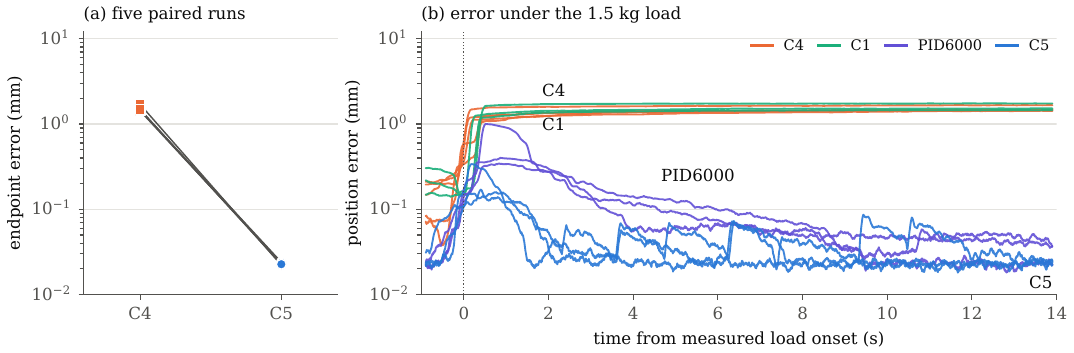}
  \caption{UR5e under the nominal 1.5\,kg load. (a)~Paired endpoints of the five
  alternated C4--C5 runs. (b)~Error norm, one line per run, smoothed over 100\,ms
  for drawing; C4 from the ablation session, the others from the later
  session with the guided attachment.}
  \label{fig:hw}
\end{figure*}

\textbf{Baselines.} C1 sags 1.46--1.74\,mm, close to C4, to which it was matched
in simulation. On the robot, C4's friction-assisted sag corresponds to about
9700\,N/m; scaling by stiffness, C1 at that value would still sag about 0.7--0.8\,mm,
but it fails the delay-margin rule used for every gain here (8\,ms delay plus
2\,ms velocity lag in simulation) and was not run. PID6000 also removes the offset, so integral action is the harder
comparison; the difference is in the transients. Over three alternated pairs
under a 35\,s load, C5 settled below 0.1\,mm sooner in all three, though this
measure resets on any single excursion and scatters. Its peak was lower in every
pair (0.20--0.40 against 0.40--1.03\,mm), as was its peak after removal
(1.04--1.23 against 1.40--2.15\,mm). At nominal 0.5 and 1.0\,kg, C5 stayed
below 0.1\,mm from 0.8--1.9\,s after onset, against 4.2--5.3\,s for PID6000
(one run 15.3\,s), with both ending at 0.021--0.036\,mm. In unloaded 50\,mm
steps, the median error at the end of a dwell is 0.023\,mm for C5 and
0.13--0.18\,mm for PID6000, measured on consecutive days. In 30\,mm steps it is
2.11\,mm for C4 and 0.031\,mm for C5 with the load attached, and 1.75 and
0.031\,mm without it. With the load attached, on one later day in alternated
order (C5, PID6000, PID6000, C5), the median dwell-end error was 0.023\,mm in
both C5 runs and 0.093--0.096\,mm for PID6000, the worst dwell 0.027 against
0.65\,mm; all four runs completed with torque limiting on at most 0.05\% of
ticks.

\textbf{Active force constraint.} With the QP's task-force box tightened to
$F_{\max}=17$\,N, a C5 run held the applied force at the limit for 82\% of the
loaded time. The force reached 17.000\,N and did not exceed it, the QP reported
no failure in at most 125 iterations, and the endpoint stayed at 0.023\,mm. The
force the load needs varies between runs, here about 15--18\,N. In a robot-side run whose need exceeded the limit, the force stayed
at 17\,N and the error fell only to 0.083\,mm, the expected cost of a binding
limit. These runs exercise the force box. The joint-torque budget did
not bind, so the evidence for the applied-torque row remains the simulation of
\S\ref{sec:tauactive}. A simulation with the robot's measured friction showed no
clear, repeatable difference between the QP force and clipped unconstrained
feedback at this load, so that comparison was not run on the robot and this
section makes no claim for the horizon.

\textbf{Hand contact.} With $F_{\max}=20$\,N, a person pushed on the tool,
on the tool side of the wrist force/torque sensor, so that every push was measured (a push
counted only if the sensor read at least 5\,N, a rule fixed before the runs).
Six pushes of 12--24\,N along $x$, $y$ and down moved the tool 0.5--3.7\,mm;
the largest error came at release (0.9--1.8\,mm), from the stored
disturbance estimate, and the tool was back below 0.1\,mm within 0.3--2.9\,s.
In a push raised slowly past the limit, the controller force reached 20\,N at a
16\,N push and stayed there; the tool moved 5.4\,mm while the push held at
20--25\,N and then stopped, joint friction carrying the excess, and it returned
below 0.1\,mm 1.05\,s after release with no error stop.

\textbf{Timing and safety.} Across the 23 runs of Table~\ref{tab:hw} and the
force-constraint run, the control-period 99th percentile is 2.12--2.18\,ms and the
maximum 5.67\,ms, the 99th percentile of compute time is at most 1.60\,ms, the
QP never failed, torque limiting acts on at most 0.05\% of ticks, and no run lost
more than one RTDE packet. A 50\,mm circle at 16\,s per lap first stopped
on the joint-speed guard during the ramp; with the orientation damping raised from
10.4 to 20 it completed twice on a later day, at 1.02--1.03\,mm RMS error, wrist speed at
most 0.20\,rad/s, and torque limiting on 0.9--1.1\% of ticks. A simulation with
that day's measured friction completes at either damping, so the change is not
shown to be the cause.

\section{Discussion and Conclusion}

We presented a compact offset-free interaction-error MPC realization:
operational-space cancellation gives a constant-$A$ predictor, force-domain
augmentation shares the optimized force channel, and an affine constraint acts
on the complete applied torque, combined with conditional impedance-equivalence,
offset-free, and stabilizability results that state exactly when each
guarantee holds --- a validated engineering interface that composes established
MPC, LPV, barrier, and passivity tools into a reusable behavior/realization
boundary.

\textbf{Where this places the contribution.} The double-integrator
representation is a general pattern for physical interaction, applying
beyond a fixed base: in selected task coordinates, after nominal
dynamic compensation, interaction reduces to a fixed linear system driven
by a single disturbance state, which makes it a predictive quantity on a
shared model. In a
model-based Physical AI stack --- perception and planning reasoning about
the world and the task, and a fast interaction interface turning that into
torque --- this paper establishes the base case of that shared
representation: a single fixed contact point, exact in force coordinates.
``Physical AI'' names that architectural placement.
We call the shared double-integrator, force-domain representation
itself \textbf{Interaction Dynamics}. The same construction has since been
developed for floating-base humanoid locomotion~\cite{caowbc}, reactive
time-domain motion planning under dynamic-obstacle interaction~\cite{caodriving},
dexterous manipulation~\cite{caodex}, and steerable-catheter tissue
interaction~\cite{caocath}, each inheriting this paper's constant-transition
structure and specializing the input matrix to its own actuation and contact
model; those extensions are reported in their own papers.

Under a shared actuator budget and torque interface, the controlled C4--C5
ablation shows that the augmented disturbance estimate alone reduces
steady-state error from 2.73 to 0.033\,mm --- an $82\times$ improvement
obtained without direct force sensing and at unchanged peak joint power, where a
stiffness-and-damping-calibrated impedance baseline stays at 2.63\,mm. The
applied-torque row holds the command inside a budget tightened to the lowest the
task can meet, at no tracking cost. The soft-barrier joint-limit avoidance, the
energy-tank passivity channel, and the disturbance-prediction bound each hold
over their respective tested and certified scope, detailed in the
corresponding sections above.

The hardware evaluation carries the central result to a physical arm with real
friction and 6\,ms actuation delay: on a UR5e the same estimator ablation takes
the held-load endpoint from 1.505\,mm to the 0.023\,mm noise level of the
encoder-based measurement, which makes it offset-free on the robot, and against tuned integral action
the estimator's advantage is in the transient rather than the final
error, and in loaded steps its dwell-end error is about a quarter of the
integral baseline's. The natural next steps are moving-reference tracking with joint-friction
compensation, a binding joint-torque constraint on the robot, and comparison
with a canonical predictive interaction controller under shared tuning.

\clearpage
\setcounter{section}{0}
\setcounter{table}{0}
\setcounter{figure}{0}
\setcounter{theorem}{0}
\setcounter{remark}{0}
\renewcommand{\thesection}{S\arabic{section}}
\renewcommand{\thetable}{S\arabic{table}}
\renewcommand{\thefigure}{S\arabic{figure}}
\renewcommand{\thetheorem}{S\arabic{theorem}}
\renewcommand{\theremark}{S\arabic{remark}}

\begin{quote}
\small\itshape
This document accompanies the main paper and contains: the full
controllability/detectability argument and LMI-based stabilizability
proof for the LPV backbone, the parameter table, and additional
experiments that supplement the main
paper's core argument --- the waypoint-hold benchmark (consistent with
Benchmark~I's conclusions), the time-varying-force horizon-prediction
study, the human-force magnitude/shape sweep, and the impedance-backbone
correction-authority ablation. Table numbers here are prefixed ``S'' and
are independent of the main paper's numbering.
\end{quote}

\section{Controllability and Detectability of the LPV Backbone}\label{sec:controllability}

\begin{remark}[Controllability and Detectability]\label{rem:feasibility}
The frozen pair in (8) of the main paper is pointwise controllable at every
fixed configuration $\rho$ in the singularity-free workspace. Since
$\Lambda(q)$ in eq.~(2) is symmetric positive-definite for all $q$ away
from kinematic singularities, $\Laminv$ is full rank, so $B_d(\rho_k)$
has full column rank 3 (via its velocity block). With $A_d$ mixing
velocity into position, the position rows of $B_d$ and $A_dB_d$ differ
by $-\Laminv\Delta t^2$ (full rank), so the controllability matrix
$\mathcal C=[B_d\mid A_dB_d]$ has rank 6 for all $\rho_k$. If
$Q\succeq0$, $R\succ0$, and $(Q^{1/2},A_d)$ is detectable, the discrete
algebraic Riccati equation (DARE) admits a stabilizing solution at every
fixed configuration, and the infinite-horizon LQR gain
\begin{equation*}
K_\infty=(R+B_d^\top P_\infty B_d)^{-1}B_d^\top P_\infty A_d
\end{equation*}
is well-defined everywhere, with the feedback convention
$\Fmpc=-K_\infty x_e$. For the augmented estimator, observability
of $[e;\dot e]$ together with full column rank of $B_d$ makes the
constant disturbance state detectable: if a mode is unobservable from
$C=[I_6\ 0]$, its $x_e$ component is zero; the augmented dynamics then
require $B_dd=0$, hence $d=0$.
\end{remark}

This is the basis for Corollary~1's DARE-solvability claim and for the
LPV-backbone stabilizability result of Sec.~\ref{sec:stab}.

\section{Common-Quadratic Stabilizability of the Scheduled Error Predictor}\label{sec:stab}

\begin{theorem}[Quadratic Stabilizability of the Scheduled Error
Predictor]\label{thm:workspace}
Consider the LPV error predictor
$x_{e,k+1}=A_dx_{e,k}+B_d(\rho_k)\Fmpck{k}$ with constant $A_d$ and
$B_d(\rho)$ as in eq.~(8) of the main paper, and suppose the
operational-space inverse-inertia stays in a compact polytope over the
operating region of interest,
$\Laminv\in\mathcal P=\operatorname{conv}\{L_1,\dots,L_V\}$ (e.g.\ the
entry-wise box with $0\prec\lambda_{\min}I\preceq\Laminv\preceq\lambda_{\max}I$)
--- a certified bound on $\Laminv$, either established analytically over
the entire singularity-free workspace or, as instantiated in
Remark~\ref{rem:instantiation} below, sampled over a specific operating
region --- with vertex input matrices
$B_v=[-\tfrac{\Delta t^2}2L_v;\,-L_v\Delta t]$ (affine in $\Laminv$, so
the polytope structure is preserved). If there exist $Q=Q^\top\succ0$ and
$Y\in\mathbb R^{3\times6}$ and a prescribed $0<\varrho<1$ such that, for
every vertex $v=1,\dots,V$,
\begin{equation}
\begin{bmatrix}\varrho^2Q & (A_dQ+B_vY)^\top\\A_dQ+B_vY & Q\end{bmatrix}\succ0,
\label{eq:lmi}
\end{equation}
then the \emph{fixed} state-feedback law $\Fmpc=Kx_e$ with $K=YQ^{-1}$
makes $V(x)=x^\top Px$, $P=Q^{-1}$, a \emph{common} quadratic Lyapunov
function: $V(x_{e,k+1})<\varrho^2V(x_{e,k})$ for
\emph{every} admissible configuration trajectory
$\{\rho_k\}\subset\mathcal P$, so this one fixed linear feedback law is
exponentially stabilizing uniformly over $\mathcal P$, not only at a
frozen configuration. It is a statement about the LPV \emph{backbone}
under a single certified feedback gain; the deployed finite-horizon MPC's
guarantee is treated in the scope remark below.
\end{theorem}

\begin{proof}
$B_d(\rho)$ is affine in $\Laminv[\rho]$, so for any $\rho$ with
$\Lambda^{-1}(\rho)=\sum_v\alpha_vL_v$ ($\alpha_v\ge0$,
$\sum_v\alpha_v=1$) we have
$A_dQ+B_d(\rho)Y=\sum_v\alpha_v(A_dQ+B_vY)$. The block matrix in
\eqref{eq:lmi} is affine in this quantity, and the positive-definite
cone is convex, so \eqref{eq:lmi} holds for all $\rho\in\mathcal P$, not
only at the vertices. Taking the Schur complement of the lower-right
block gives
$(A_dQ+B_d(\rho)Y)^\top Q^{-1}(A_dQ+B_d(\rho)Y)\prec\varrho^2Q$.
Writing $A_\text{cl}(\rho)=A_d+B_d(\rho)K$ with $K=YQ^{-1}$ and
pre/post-multiplying by $Q^{-1}=P$ yields
$A_\text{cl}(\rho)^\top PA_\text{cl}(\rho)-\varrho^2P\prec0$, for
all $\rho\in\mathcal P$. Hence $V(x)=x^\top Px$ decreases along every
trajectory regardless of how $\rho_k$ varies, which is exponential
stability of this fixed-gain, disturbance-free LPV closed loop.
\end{proof}

This is the polytope-level counterpart of the per-configuration DARE
solvability of Remark~\ref{rem:feasibility} and the continuous scheduled
law of Corollary~1 of the main paper: it certifies that the
interaction-error \emph{representation} --- enabled by the constant
$A_d$, which confines all configuration variation to the affine
$B_d(\rho)$ and makes \eqref{eq:lmi} a finite vertex program --- is
quadratically stabilizable by a single Lyapunov function over
$\mathcal P$, under one fixed certified gain.

\textbf{Scope.} Theorem~\ref{thm:workspace} establishes exponential
stability of the \emph{nominal predictive backbone under one fixed,
LMI-certified linear feedback gain} --- the disturbance-free LPV closed
loop $(A_d,B_d(\rho))$, with the input constraint inactive. The remaining
elements of the deployed controller are covered by their own results: the
disturbance estimator by Theorem~2 of the main paper (offset-free
tracking under the Kalman estimator); the finite-horizon receding-horizon
re-solve, whose realized policy at each step is the QP's actual
minimizer, realizes \emph{a} member of this same class of linear state
feedback (Theorem~1/Corollary~1 of the main paper); and the safety and
passivity filters by Theorem~3 and Proposition~2 of the main paper. The
complete controller's closed-loop guarantee is therefore this
\emph{collection} of per-element results.

\begin{remark}[Rate-independence and the frozen horizon]
The certificate uses a \emph{single, parameter-independent} Lyapunov
matrix $P=Q^{-1}$ shared by all vertices. A common-$P$ quadratic-stability
certificate is robust to \emph{arbitrary} time variation of
$\rho_k\in\mathcal P$ --- including arbitrarily fast changes --- with no
assumption on the parameter-variation rate $\|\rho_{k+1}-\rho_k\|$: the
Lyapunov decrease holds at every vertex and hence, by convexity of
\eqref{eq:lmi} in $\Laminv$, for any $\rho_k\in\mathcal P$ regardless of
how it moves, \emph{provided the certified fixed $K$ is what is actually
applied}. This gives a \emph{structural reason} why the
frozen-$\Gamma$ approximation used online (built from the current
$B_d(q_k)$ and held over the $N$-step horizon) is unlikely to be
destabilizing: the gap between $B_d(q_k)$ and the true future
$B_d(q_{k+i})$ is a bounded variation \emph{inside} $\mathcal P$, and a
fixed-$K$ closed loop over $\mathcal P$ tolerates that
variation regardless of its rate. This motivates --- without formally
establishing --- that the deployed receding-horizon policy inherits the
same robustness: the frozen horizon shapes only the \emph{predicted}
trajectory, while the receding-horizon re-solve keeps the
\emph{applied} first move matched to the current $\rho_k$.
\end{remark}

\begin{remark}[Numerical instantiation of $\mathcal P$]\label{rem:instantiation}
For the FR3, sampling $\Laminv$ along the main paper's \S V circular
benchmark (4700 samples) gives
$\operatorname{eig}\Laminv\in[0.063,0.347]\ \text{kg}^{-1}$ (task
inertia 2.9--15.9\,kg). The resulting entry-wise box has 64 corners in
general, but not every corner of an independent per-entry box need be
positive-definite; we discard any that are not (none were, here --- all
64 corners were confirmed SPD, minimum eigenvalue $\approx0.048\
\text{kg}^{-1}$ across vertices, comfortably above the regularization
floor). Direct entry-wise auditing reports zero containment violation for all
4700 samples. This instantiates $\mathcal P$ over the \emph{sampled operating
region traversed by the benchmark trajectory}; extending it to
other trajectories or the full workspace is a matter of re-sampling or an
analytic eigenvalue bound. Over this $\mathcal P$
($\Delta t=10$\,ms), the minimum-condition-number common-$P$ certificate
at guaranteed rate $\rho\le0.996$ is feasible, returning a
well-conditioned common $P$ ($\operatorname{cond}P\approx2.2$) with
spectral radius $\le0.996$ and strictly negative Lyapunov increment
($\max\operatorname{eig}(A_\text{cl}^\top PA_\text{cl}-P)\approx-5\times10^{-3}$)
at every vertex. For the exact-ZOH model, CLARABEL reports
\texttt{optimal}; the target-rate residual
$\max\operatorname{eig}(A_\text{cl}^\top PA_\text{cl}-0.996^2P)$ is
$-8.3\times10^{-7}$ and the minimum eigenvalue of the LMI block is
$9.1\times10^{-7}$. These checks apply \emph{for the certified fixed
$K=YQ^{-1}$}; the
realized MPC, with its larger tracking weights and receding-horizon
policy, is observed empirically to converge faster. This establishes the
qualitative backbone-stabilizability claim rigorously over the sampled
region rather than empirically, and is reproducible from the released
script (\texttt{lmi\_workspace\_stability\_probe.py}).
\end{remark}

\section{Parameters}\label{sec:params}

Every numeric constant used across the experiments of the main paper
(\S VI) and this supplement is consolidated below, cross-checked against
the released files \path{simulation/impedance_mpc.py},
\path{simulation/phri.py}, and
\path{cloud_verify/lib/fr3_hardware_interface.py} rather than
restated from memory. Unless a table or figure caption states otherwise,
these are the values used throughout.

\begin{table}[!t]
\caption{Controller and Estimator Parameters}
\label{tab:params}
\renewcommand{\arraystretch}{1.0}
\footnotesize
\setlength{\tabcolsep}{3pt}
\begin{tabular}{p{2.3cm}p{1.2cm}p{2.5cm}}
\toprule
\textbf{Parameter} & \textbf{Symbol} & \textbf{Value}\\
\midrule
Prediction horizon & $N$ & 10\\
MPC sample time (100\,Hz) & $\Delta t$ & 10\,ms\\
MPC sample time (500\,Hz) & $\Delta t$ & 2\,ms\\
Inner-loop rate & --- & 1\,kHz\\
Position stage weight & $Q_\text{pos}$ & $2\times10^4 I_3$\\
Velocity stage weight & $Q_\text{vel}$ & $50\,I_3$\\
Terminal cost scale & $\gamma$ & 5\\
Effort weight & $R$ & $10^{-6}I_3$\\
Corrective-force bound & $F_\text{max}$ & 150\,N\\
FR3 torque limits & $\tau_\text{max}$ & $87\!\times\!4;\ 12\!\times\!3$\,Nm\\
$\Lambda^{-1}$ regularization & $\sigma$ & $10^{-6}$\\
Kalman process noise & $Q_d$ & 10.0 (2.0 in C6--C7)\\
Kalman obs.\ noise (pos.) & --- & $10^{-3}$\\
Kalman obs.\ noise (vel.) & --- & $10^{-2}$\\
Orientation stiffness & $K_\text{rot}$ & 20\,Nm/rad\\
Orientation damping & $D_\text{rot}$ & 6\,Nm$\cdot$s/rad\\
Null-space centering stiffness & $k_\text{null}$ & 10\,Nm/rad\\
Null-space damping & $d_\text{null}$ & 2\,Nm$\cdot$s/rad\\
Null-space barrier gain & $k_b$ & 50\,Nm\\
Barrier activation threshold & $\eta$ & 0.10\\
Workspace-projection gain & $k_\text{ws}$ & $5\times10^{-4}$ m/(Nm/rad)$^\dagger$\\
Workspace-projection cap & $p_\text{max}$ & 0.06\,m\\
Workspace-projection reg. & $\epsilon_r$ & $10^{-8}$\\
Backbone stiffness (C8) & $K_\text{bb}$ & 300\,N/m\\
Backbone damping ratio (C8) & $\zeta_\text{bb}$ & 1.0 (nominal)\\
CBF safety margin & $\epsilon$ & 0.05\,rad\\
CBF contraction rate & $\alpha$ & 0.5\\
CBF slack penalty weight & $w_s$ & $10^{8}$\\
Tank initial energy & $E_0$ & 20.0\,J\\
Tank lower bound & $E_\text{min}$ & 0.05\,J\\
Tank upper bound & $E_\text{max}$ & 50.0\,J\\
Human-recharge fraction & $\gamma$ (tank) & 0.0\\
OSQP tolerances & abs., rel. & $10^{-6}$\\
OSQP iteration cap & --- & 4000\\
Circular benchmark $R,T$ & --- & 0.12\,m, 8\,s\\
Human step force & $F_h$ & 15\,N, $t\in[3,6]$\,s\\
Reference ramp time & $T_\text{ramp}$ & 1.5\,s\\
FR3 neutral pose & $q_0$ &
$\begin{array}{c}[0,-0.785,0,\\-2.356,0,1.571,\\0.785]\end{array}$ rad\\
\bottomrule
\end{tabular}
\end{table}

Two parameter sets are estimator/CBF-internal: the Kalman noise
covariances and the CBF/tank constants (the latter exercised in the
\texttt{cloud\_verify} hardware-path suite rather than the direct
benchmark scripts used for the main paper's tables).

$^\dagger$Benchmarks I and II in the main paper are run with
$k_\text{ws}=0$: the driver scripts
for those two circular/waypoint benchmarks disable workspace projection
so its effect does not confound the controller-to-controller
comparison. The $5\times10^{-4}$ default above is what the boundary
test in the main paper actually uses, which is where the
mechanism is exercised and discussed.

\section{Focused Fair-Comparison Protocol}\label{sec:fair}

The focused comparison in the main paper separates calibration from testing.
Calibration uses one 8\,s cycle; after selecting parameters, every reported
metric is recomputed in a newly reset environment over three cycles (24\,s).
No test-cycle metric enters parameter selection. The
stiffness-and-damping-calibrated impedance first uses $K_0=\|F_h\|/\|e_{\rm ss,C4}\|=5496$\,N/m, then evaluates
the Cartesian product
\[
K/K_0\in\{0.8,1.0,1.2\},\qquad
\zeta\in\{0.7,1.0,1.3\}.
\]
The score is contact-onset trajectory RMSE relative to C4 plus absolute
steady-state-error mismatch; $K=5496$\,N/m, $\zeta=1.3$ is selected. The
anti-windup PI grid is
\[
K_i\in\{120,300,700\},\qquad
I_{\max}\in\{0.02,0.06,0.12\}\ \text{m\,s},
\]
with score $e_{\rm ss}+0.25e_{\rm RMS,contact}+0.01r_{\rm sat}$; it selects
$K_i=120$ and $I_{\max}=0.02$\,m\,s. The complete 18-run calibration ledger,
selected scores, test metrics, QP counts, and saturation diagnostics are stored
in \path{simulation/fair_offset_free_comparison.json} and regenerated by
the identically named Python script.

The measured-force comparator uses the stiffness-and-damping-calibrated impedance and receives
the exact applied simulation wrench through a first-order filter
($T_f=50$\,ms), commanding $-J_v^\top\hat F_h$ in the same Cartesian force
channel through which $F_h$ enters. It reaches 1.36\,mm total steady-state
error (0.42\,mm on the force axis), confirming that current-step force
cancellation rejects a constant load; unlike C5 it requires a direct
force measurement. Both MPC variants have zero
failed solves in 2400 updates, and no raw command exceeds a joint limit in
any of the six test runs. This focused script isolates the controller
comparison from the CBF, tank, and torque-rate filters.

\section{Four-Plane Circle Tracking}\label{sec:planes}

C5 (MPC + Kalman, 100\,Hz QP, 1\,kHz inner) is evaluated on four 3-D
circle planes, complementing the XZ-plane result of Benchmark~I.
Table~\ref{tab:planes} reports RMSE after the 1.5\,s ramp.

\begin{table}[!t]
\caption{Free-Space Circle Tracking, $R=12$\,cm, $T=8$\,s}
\label{tab:planes}
\renewcommand{\arraystretch}{1.0}
\centering
\setlength{\tabcolsep}{4pt}
\footnotesize
\begin{tabular}{@{}lccc@{}}
\toprule
\textbf{Plane} & \textbf{IMP RMSE [mm]} & \textbf{MPC RMSE [mm]} & \textbf{Improvement}\\
\midrule
XZ sagittal  & 25.21 & 0.31 & $\times$82\\
XY horizontal & 25.61 & 0.33 & $\times$77\\
YZ frontal   & 18.16 & 0.26 & $\times$69\\
XZ$\to$XY tilted & 26.42 & 0.31 & $\times$84\\
\bottomrule
\end{tabular}
\end{table}

\textbf{Analysis.} The MPC maintains sub-0.35\,mm RMSE across all
planes, demonstrating that the $\Lam$-adaptive compliance in the QP
(eq.~9 of the main paper) handles the off-diagonal operational-space
inertia coupling without plane-specific gain retuning --- the controller was tuned
and validated on the XZ plane alone (Benchmark~I), and the other three
planes are a held-out generalization check on that single tuning.

\section{Benchmark II --- Reach-and-Hold under Human Push}\label{sec:bm2}

\textbf{Scenario.} The robot navigates a triangle of three waypoints (A,
B, C) in one lap. At each waypoint, 0.8\,s after arrival, a directionally
varied 15\,N push fires for 2\,s. The robot must recover and dwell
within 35\,mm of the target for 1\,s before advancing (3 push events, one
per waypoint).

\begin{figure}[!t]
\centering
\includegraphics[width=\columnwidth]{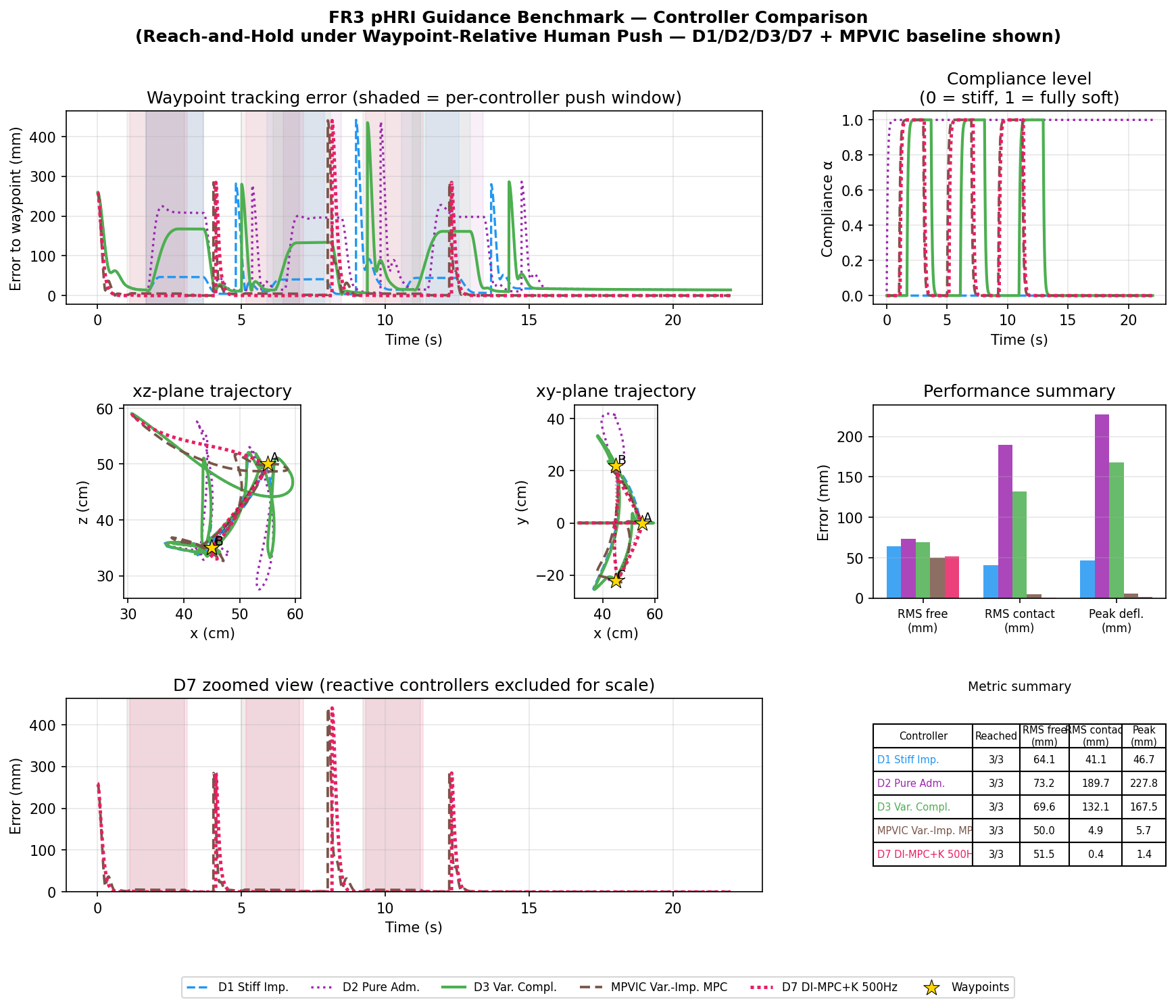}
\caption{Benchmark II --- reach-and-hold under directionally varied
push. For visibility the plot shows D1/D2/D3/D7 (corresponding to
G1/G2/G3/G7) plus the predictive variable-impedance MPVIC baseline;
Table~\ref{tab:bm2} reports the full ablation. The final D7 controller
reaches every waypoint with sub-millimeter contact-window deflection.}
\label{fig:bm2}
\end{figure}

\begin{table}[!t]
\caption{Benchmark II: Waypoint-Hold Under Directionally Varied Push}
\label{tab:bm2}
\renewcommand{\arraystretch}{1.0}
\footnotesize
\setlength{\tabcolsep}{3pt}
\begin{tabular}{lcccc}
\toprule
\textbf{Controller} & \textbf{Waypts} & \textbf{RMS free} & \textbf{RMS cont.} & \textbf{Peak}\\
\midrule
G1 -- Stiff Impedance    & 3/3 & 64.1 & 41.1 & 46.7\\
G2 -- Pure Admittance    & 3/3 & 73.2 & 189.7 & 227.8\\
G3 -- Variable Compl.    & 3/3 & 69.6 & 132.1 & 167.5\\
MPVIC -- Var.-Imp. MPC   & 3/3 & 50.0 & 4.9 & 5.7\\
G4 -- MPC 100\,Hz        & 3/3 & \textbf{47.1} & 1.9 & 2.6\\
G5 -- MPC+K 100\,Hz      & 3/3 & 48.0 & 0.6 & 2.4\\
G6 -- MPC 500\,Hz        & 3/3 & 51.9 & 0.8 & \textbf{1.0}\\
G7 -- MPC+K 500\,Hz      & 3/3 & 51.5 & \textbf{0.37} & 1.42\\
\bottomrule
\end{tabular}
\end{table}

\textbf{Analysis.} The static-waypoint hold confirms the same separation
seen in Benchmark~I. All four MPC variants (G4--G7) reach every
waypoint with order-of-magnitude lower contact-window deflection than
the reactive baselines: against the stiff-impedance baseline (G1), the
MPC cuts contact-window RMS from 41.1\,mm to 1.9\,mm (G4) and peak
deflection from 46.7\,mm to 2.6\,mm. Adding the Kalman augmentation
improves the contact metrics --- G5 lowers contact-window
RMS from 1.9\,mm to 0.6\,mm and slightly reduces peak (2.6$\to$2.4\,mm)
--- so unlike a higher impedance gain it carries no peak penalty; free-motion
RMS is essentially unchanged (47.1 vs.\ 48.0\,mm). Raising
the QP rate to 500\,Hz then sharpens contact-window RMS further: G7
attains 0.37\,mm. Under the rate-proportional $Q_w$ scaling of
Table~\ref{tab:params}, however, it no longer has the lowest peak error
--- G6's 1.0\,mm undercuts G7's 1.42\,mm, mirroring the same pattern seen
in Table~\ref{tab:bm1}. The
pure-admittance and variable-compliance baselines (G2, G3) yield by
design, producing the large 190--132\,mm contact-window deflections
expected of intentional compliance. The predictive variable-impedance
baseline (MPVIC) behaves oppositely to the reactive variable-compliance
G3: instead of softening, it stiffens predictively to reject the push,
cutting contact-window RMS to 4.9\,mm, yet it stays a further order of
magnitude above the offset-free G5 (0.6\,mm) --- the same non-offset-free
residual seen in Benchmark~I. As in Benchmark~I, MPVIC here is our own
discrete-stiffness scheduler in the family of
\cite{liu2025model,roveda2019optimal}.

\section{Time-Varying Interaction and Horizon-Prediction Validity}\label{sec:tv}

The benchmarks in the main paper apply constant or step forces. Because
the framework targets \emph{interaction} dynamics broadly, we
additionally exercise a \emph{time-varying} human force and measure the
$N$-step disturbance-prediction RMS $\varepsilon_N$ of
Section~III-D of the main paper --- the quantity that bounds how well the
flat random-walk prediction holds over the horizon. To reduce
configuration-driven variation, the robot
holds a \emph{static} target while a sinusoidal push
$F_z=-A\sin(2\pi ft)$ of fixed amplitude $A=12$\,N and increasing
frequency $f$ is applied, sweeping the disturbance rate
$A\,2\pi f$, the injected sinusoid's peak derivative. Because the theorem
concerns the aggregate disturbance, the table uses the reconstructed-log rate
$L_{d,\mathrm{emp}}=\max_k\|d_{k+1}-d_k\|_2/\Delta t$.
The controller is DI-MPC + Kalman (C5, 100\,Hz
QP).

\textbf{Direct ground-truth metric.} At each manager instant we compute
the simulator acceleration from the applied, torque-limited input and the
rigid-body dynamics, then reconstruct
\begin{equation*}
d_{\mathrm{acc},k}=\ddot e_k+\Lambda^{-1}(q_k)\Fmpck{k},\qquad
d_k=-\Lambda(q_k)d_{\mathrm{acc},k}.
\end{equation*}
For the flat prediction made at origin $k$, the direct metric is
$\varepsilon_i=\mathrm{RMS}_k\|d_{k+i}-\dhat_{k|k}\|_2$. We also report
$e_K=\mathrm{RMS}_k\|d_k-\dhat_{k|k}\|_2$. By Minkowski's inequality,
the pointwise rate assumption in the main paper's horizon-prediction bound
gives $\varepsilon_N\le e_K+L_dN\Delta t$; Table~\ref{tab:tv} evaluates it
with $L_d=L_{d,\mathrm{emp}}$. Both $e_K$ and $L_{d,\mathrm{emp}}$ are
estimated from this same logged trajectory, so what follows checks the
bound's internal arithmetic against its own inputs. Origins are restricted
to the settled portion of the force window and all $N$ target samples
remain inside that window (704 origins per row).

\begin{table}[!t]
\caption{Time-Varying Force: Direct Ground-Truth Prediction Error and Tracking}
\label{tab:tv}
\renewcommand{\arraystretch}{1.0}
\footnotesize
\setlength{\tabcolsep}{3pt}
\begin{tabular}{cccccc}
\toprule
$L_{d,\mathrm{emp}}$ [N/s] & $e_K$ & $\varepsilon_1$ & $\varepsilon_N$ & Bound & RMS [mm]\\
\midrule
1.01  & 1.71 & 1.71 & 1.70 & 1.81 & 0.07\\
7.58  & 0.53 & 0.55 & 0.86 & 1.28 & 0.12\\
15.15 & 2.05 & 2.11 & 2.74 & 3.57 & 0.26\\
31.14 & 4.20 & 4.33 & 5.76 & 7.32 & 0.54\\
61.09 & 5.09 & 5.39 & 8.50 & 11.19 & 1.04\\
\bottomrule
\end{tabular}
\end{table}

\textbf{Analysis.} The directly measured $N$-step error stays below the
RMS bound in every tested row. In the constant injected-force row,
$L_{d,\mathrm{emp}}=1.01$\,N/s captures the small variation of the other
aggregate residual terms, giving a 1.81\,N bound versus
$\varepsilon_N=1.70$\,N. The estimator error in that row is nonzero: the
mean error-vector norm is 0.30\,N and zero-mean fluctuation RMS is
1.68\,N, so the 1.71\,N RMS error is primarily fluctuation rather than steady
bias, while closed-loop position RMS remains 0.07\,mm. As frequency increases,
$\varepsilon_N$ rises from 0.86 to 8.50\,N while remaining below the corresponding
1.28--11.19\,N bounds. Tracking degrades gradually: it remains below
0.6\,mm through $L_{d,\mathrm{emp}}=31.14$\,N/s and reaches 1.04\,mm at
61.09\,N/s.
These data directly check that bound for the stated simulator protocol.
For a
\emph{structured} high-rate force (e.g.\ physiological tremor) the
harmonic-internal-model extension of Section~III-D would predict that
component forward exactly rather than extrapolate it flat. The
experiment and its machine-readable output are reproduced by
\texttt{time\_varying\_experiment.py} and
\texttt{time\_varying\_ground\_truth\_results.json}.

\section{Human-Force Magnitude and Shape Sweep}\label{sec:forcesweep}

Benchmark~I and \S\ref{sec:tv} fix the force magnitude at 12--15\,N.
Because the offset-free property (Theorem~2 of the main paper) is
independent of the force amplitude, we sweep a sustained (step) push
$F_z\in\{5,10,15,20,25\}$\,N on the circular tracking task and report the
steady-state and peak tip deflection for classical impedance (D1) and
the proposed offset-free controller (D7, DI-MPC + Kalman, 500\,Hz); a
shape sweep at 15\,N (step / linear ramp / 1\,Hz sinusoid) complements
the time-varying study of \S\ref{sec:tv}.

\begin{table}[!t]
\caption{Human-Force Magnitude Sweep (SS / peak deflection, mm)}
\label{tab:fsweep}
\renewcommand{\arraystretch}{1.0}
\footnotesize
\setlength{\tabcolsep}{3pt}
\begin{tabular}{lccccc}
\toprule
$F_z$ [N] & 5 & 10 & 15 & 20 & 25\\
\midrule
D1 Imp.\ -- SS   & 26.7 & 33.3 & 44.7 & 58.2 & 72.7\\
D1 Imp.\ -- peak & 29.2 & 37.3 & 52.0 & 71.4 & 90.6\\
\textbf{D7 -- SS}   & \textbf{0.031} & \textbf{0.031} & \textbf{0.031} & \textbf{0.031} & \textbf{0.031}\\
\textbf{D7 -- peak} & \textbf{0.64} & \textbf{0.98} & \textbf{1.48} & \textbf{1.97} & \textbf{2.47}\\
\bottomrule
\end{tabular}
\end{table}

\textbf{Analysis.} Classical impedance deflects in proportion to the
force --- the $e_\infty=K_d^{-1}F_h$ bias combined with the dynamic
tracking error grows from 27\,mm at 5\,N to 73\,mm at 25\,N --- whereas
the proposed controller holds an approximately \textbf{0.031\,mm
steady-state deflection independent of magnitude over the tested
5--25\,N unsaturated range}: the integrating disturbance state cancels
the sustained force in this range, so accuracy no longer trades against
the human's effort. The first-contact peak grows
more substantially (0.64 $\to$ 2.47\,mm across 5--25\,N), scaling with the
force step before $\dhat$ converges. Under the three 15\,N shapes the
proposed controller keeps contact-window RMS at 0.33\,mm (step), 0.21\,mm
(ramp), and 0.87\,mm (1\,Hz sinusoid) versus 40.9 / 26.3 / 33.0\,mm for
classical impedance; the sinusoid is the worst case, consistent with
\S\ref{sec:tv} --- the flat $\dot d=0$ model lags a fast-varying force,
which the harmonic internal model of Section~III-D of the main paper
would recover. Reproduced by \texttt{force\_sweep.py}.

\begin{figure}[t]
\centering
\includegraphics[width=\columnwidth]{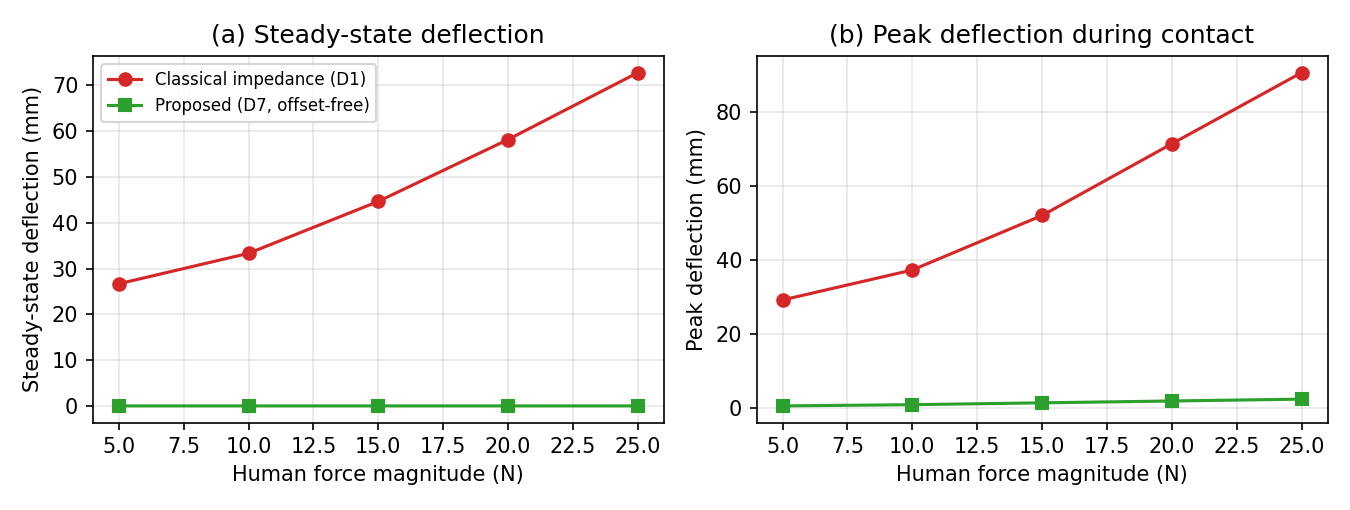}
\caption{Human-force magnitude sweep: steady-state (a) and peak (b) tip
deflection versus a sustained push, for classical impedance (D1, red)
and the proposed offset-free controller (D7, green). The proposed
steady-state deflection is invariant to force magnitude (0.031\,mm),
while classical impedance grows linearly.}
\label{fig:fsweep}
\end{figure}

\section{Correction-Authority Robustness (Impedance-Backbone Ablation)}\label{sec:backbone}

Every controller in the main paper places the entire task-space
correction inside one box-constrained decision variable, $\Fmpc$: if it
is driven toward zero --- a tight $F_\text{max}$ bound, or a horizon step
whose predicted torque is unrealizable --- the commanded torque degrades
to the feedforward term $\tau_\text{ff}$ alone, i.e.\ zero corrective
stiffness. We evaluate an architectural ablation, \textbf{C8}, that
decouples baseline corrective stiffness from the optimizer: a fixed,
positively damped impedance law
\begin{equation*}
F_\text{bb}=K_\text{bb}e+D_\text{bb}\dot e,\qquad K_\text{bb}=300\text{ N/m},
\end{equation*}
is commanded on every MPC update independent of the QP output
($D_\text{bb}=2\zeta_\text{bb}\sqrt{K_\text{bb}}$, a nominal
unit-effective-mass tuning; since $\Laminv$ is generally anisotropic, the
damping is approximately modal-critical), and the QP shapes a
bounded additional correction $\Fmpc$ ($\|\Fmpc\|_\infty\le F_\text{max}$)
predicted through the backbone dynamics linearized at the current
configuration, $A_\text{cl}=A_d+B_d(\rho_k)G_\text{bb}$. If the QP's
output is driven to exactly zero for any reason, the commanded controller
retains this non-zero corrective stiffness instead of reverting to bare
feedforward. This gives QP-independence while keeping the actuator-limit
accounting intact: the \textbf{total commanded torque is
$\tau=\tau_\text{base}+J_v^\top(F_\text{bb}+\Fmpc)$}, so $F_\text{bb}$
enters every torque-feasibility check,
and C8 additionally applies the main paper's eq.~(9b) row at
\emph{every} horizon step rather than only the first, frozen at the
current configuration:
\begin{align*}
-\tau_\text{max}&\le\tau_\text{base}(0)+J_v^\top(q_0)
\bigl(F_\text{bb}+\Fmpck{i}\bigr)\\
&\le\tau_\text{max},\qquad i=0,\ldots,N-1,
\end{align*}
the frozen-Jacobian horizon-wide extension of (9b), specific to C8 (the
default C1--C7 formulation applies (9b) at the first step only). Freezing
$J_v(q_0)$, $\tau_\text{base}(0)$ across the
horizon is a local approximation of future
torque feasibility. With this in place, the evidence below supports
robustness to loss of \textbf{additive QP correction authority}
specifically.

\textbf{Scenario and protocol.} Identical to Benchmark~I: 12\,cm
circular reference, 15\,N step push, 3 cycles / 24\,s, contact-window
metrics averaged over all three push events. We sweep the
corrective-force bound $F_\text{max}\in\{150,20,5,1,0\}$\,N, holding
every other parameter fixed, emulating progressively severe loss of
additive correction authority; $F_\text{max}=0$\,N is the limiting case
in which the QP contributes nothing.

\textbf{Results.} Under the normal protocol ($F_\text{max}=150$\,N), C8
matches C7 to measurement precision: RMS/contact/peak/SS is
0.31/0.33/1.48/0.031\,mm versus 0.31/0.33/1.49/0.031\,mm. The stress sweep
is the relevant comparison:

\begin{table}[!t]
\caption{Correction-Authority Ablation: $F_\text{max}$ Sweep, RMS /
Peak (mm)}
\label{tab:backbone}
\renewcommand{\arraystretch}{1.0}
\footnotesize
\setlength{\tabcolsep}{3pt}
\begin{tabular}{lccccc}
\toprule
$F_\text{max}$ [N] & 150 & 20 & 5 & 1 & 0\\
\midrule
C7 -- RMS cont. & 0.33 & 0.33 & 319.0 & 411.1 & 419.8\\
C7 -- peak      & 1.49 & 1.49 & 408.4 & 547.2 & 553.8\\
\textbf{C8 -- RMS cont.} & \textbf{0.33} & \textbf{0.33} & \textbf{40.9} & \textbf{70.3} & \textbf{79.2}\\
\textbf{C8 -- peak}      & \textbf{1.48} & \textbf{1.48} & \textbf{54.2} & \textbf{100.4} & \textbf{113.1}\\
\bottomrule
\end{tabular}
\end{table}

\textbf{Analysis.} At $F_\text{max}\ge20$\,N there is enough headroom
that both controllers are unaffected and C8 costs nothing relative to
C7 --- the backbone's restoring term is exactly what the unconstrained
QP would already have produced. Below that, C7's contact-window RMS
jumps by roughly three orders of magnitude (0.33 $\to$ 420\,mm) the
moment the corrective force can no longer supply the needed authority:
with $\Fmpc$ forced toward zero, the commanded torque approaches bare
feedforward against the sustained push. \textbf{C8 degrades more gracefully
over the tested 24\,s protocol}, reaching 41--79\,mm even in the
$F_\text{max}=0$ case --- a 5.3--7.8$\times$ smaller error than C7 across the
curtailed range. The impedance-equivalence and offset-free
statements of Theorems~1--2 of the main paper are proved for the standard
(non-backbone) realization; C8's additive term retains the same
offset-free mechanism.
This deployment ablation is distinct from the energy tank, which addresses
coupled-system stability rather than QP-authority loss. Reproduced by
\texttt{stable\_backbone\_comparison.py --n-cycles 3}.

\section{Active Torque-Constraint Stress Test}\label{sec:tauactive}

The stress test above tightens the \emph{force} box $F_\text{max}$; every
benchmark elsewhere in this paper keeps $\tau_\text{max}$ at its full FR3
value, where the applied-torque row (eq.~(9b) of the main paper) never
binds. This section tightens $\tau_\text{max}$ directly via a multiplicative
scale and asks what the torque row costs and buys once it can
genuinely bind, comparing four otherwise-identical 100\,Hz controllers on
the Benchmark~I protocol. Along the unconstrained run, joint~2 reaches 34\% of its
limit, and joint~6 reaches 52\% for a moment at $t=0$, when the reference velocity
steps; scale 0.34 is therefore the lowest budget at which the task itself can stay
within the limits after that first instant:

\begin{itemize}[leftmargin=*,itemsep=1pt,topsep=2pt]
\item \textbf{C5} --- the row embedded as in the main paper: on
infeasibility, the released implementation returns zero correction
(\S\ref{sec:layer2}).
\item \textbf{Unconstrained} --- the row omitted entirely; the plant
still clips physically (\texttt{FR3MuJoCoEnv.apply\_torque}), so this
isolates whether the QP's own optimum accounts for $\tau_\text{max}$.
\item \textbf{Observer} --- the fixed-gain baseline of
\S\ref{sec:experiments} under the same tightened budget, using the same
torque row as C5 but a single-step, non-predictive QP.
\item \textbf{Soft} --- the same row relaxed with a slack variable
$s\ge0$ on each joint row,
penalized $\tfrac12\rho\|s\|^2$ in the QP cost instead of the hard
zero-correction fallback, so the QP is always feasible. Implemented as an
extended decision variable $[U;s]\in\mathbb R^{37}$ (one slack per joint
row) solved in the same QP, with $\rho=1$.
\end{itemize}

\begin{table}[!t]
\caption{Torque-Budget Sweep}
\label{tab:tauactive}
\renewcommand{\arraystretch}{1.0}
\footnotesize
\setlength{\tabcolsep}{3pt}
\resizebox{\columnwidth}{!}{%
\begin{tabular}{lcccc}
\toprule
$\tau_\text{max}$ scale & 1.0 & 0.34 & 0.32 & 0.30\\
\midrule
\multicolumn{5}{l}{\emph{(a) RMS in contact / outside contact [mm]}}\\
C5            & 0.54 / 0.44 & 0.54 / 0.44 & 73.1 / 107 & $157$--$487$ / $349$--$640^\ast$\\
Unconstrained & 0.54 / 0.44 & 0.54 / 0.44 & 0.54 / 0.44 & 0.54 / 0.44\\
Observer      & 0.54 / 0.45 & 0.54 / 0.45 & 0.54 / 47.3 & 1.21 / 102\\
Soft          & 0.54 / 0.44 & 0.54 / 0.44 & 0.53 / 47.9 & 1.18 / 158\\
\midrule
\multicolumn{5}{l}{\emph{(b) Largest command above the tested budget [Nm]}}\\
C5            & 0 & 0 & 0.82 & diverges$^\ast$\\
Unconstrained & 0 & 2.17 & 2.41 & 3.45\\
Observer      & 0 & 0 & 0.12 & 0.49\\
Soft          & 0 & 0.04 & 2.22 & 19.4\\
\bottomrule
\end{tabular}}

\vspace{2pt}
\noindent\parbox{\columnwidth}{\footnotesize $^\ast$Range over six runs whose
orientation set-point differs by ${\le}2\times10^{-9}$\,rad; the hard row is infeasible
on 17--55\% of solves and the outcome is chaotic. The other three rows are
identical across these runs to three decimals, and every run at scales 1.0, 0.34
and 0.32 has zero failed solves.}
\end{table}

\textbf{Analysis.} At full scale all four are
indistinguishable and no row binds. At scale 0.34, the lowest budget the task can
meet, the row binds only in the first instants, and it costs nothing: C5 stays inside the
budget with zero failed solves and the same contact and free-motion error, its largest
error rising from 2.52 to 2.63\,mm at the start. Unconstrained
tracks equally well by exceeding the budget by up to 2.17\,Nm. This is the row's
positive case.

Below that, no controller both tracks and respects the budget. At scale 0.32 the
contact-window error alone does not show this: Observer and Soft stay at 0.54\,mm in
contact, while their error outside
contact rises to 47--48\,mm, against 0.44\,mm for Unconstrained, which simply
ignores the budget. Hard-row C5 stays feasible but loses more, with 73\,mm in contact and
107\,mm outside. Soft does not recover the lost tracking here; its slack (up to 2.22\,Nm,
active on 45\% of solves) lets it overrun the budget by almost as much as
Unconstrained while keeping Observer's error. None of the four reaches the robot's
full torque limit at this scale.

At scale 0.30 the hard row becomes infeasible on 17--55\% of solves (398--1319 of
2400, across six set-points), and the zero-correction fallback
drives the arm far from the path. Which solves fail is chaotic, so only the range is
meaningful. Soft avoids that failure but draws up to 19.4\,Nm of slack with 158\,mm of
free-motion error, and at scale 0.2 its slack diverges to non-physical magnitudes;
this is outside the regime where the relaxation is useful.
Reproduced by \texttt{torque\_budget\_final\_table.py} and
\texttt{soft\_constraint\_diagnostic.py} (slack/excess instrumentation).

\section{Low-Stiffness Transient Compliance and Recovery}\label{sec:lowk}

Every comparison elsewhere in this paper realizes a high closed-loop
stiffness (5496\,N/m calibrated per Table~\ref{tab:fair}, or C4/C5's own
Riccati-image gain), so peak deflection under the 15\,N push is only a
few mm --- closer to a stiff trajectory tracker with a disturbance
observer than to a controller that visibly yields to contact and then
recovers. This section tests the architecture's other end of the design
space: a deliberately low prescribed stiffness $K_d=100$\,N/m (matching
C2's admittance virtual-spring stiffness), realized through a
cost-shaped QP variant that directly penalizes deviation of the force
sequence from the prescribed law's own rollout ($H=\bar R$,
$h=-\bar RU_\text{imp}$ in the main paper's condensed-QP notation):
where Theorem~\ref{thm:equiv}'s general Riccati image realizes whatever
gain the LQ cost implies, this shaping makes the
unconstrained minimizer exactly $\Fmpc=K_de+D_d\dot e-\hat d$ for any
chosen $K_d,D_d$, with and without the $-\hat d$ term. $D_d=2\zeta\sqrt{K_d}$ assumes a unit
effective mass; $\Laminv$ is generally anisotropic (the same caveat
already noted for the C8 backbone's $D_\text{bb}$ above), so $\zeta=1$ is
only nominal critical damping and produced a visibly underdamped,
oscillating response at this low $K_d$ when traced directly (raw error
vs.\ time, not just the summary RMS); a small sweep selects $\zeta=3$ for
clean, non-oscillatory settling.

\begin{table}[!t]
\caption{Low-Stiffness (100\,N/m) Yield-Then-Recover}
\label{tab:lowk}
\renewcommand{\arraystretch}{1.0}
\footnotesize
\setlength{\tabcolsep}{3.5pt}
\begin{tabular}{lrrrr}
\toprule
\textbf{Controller} & \textbf{Peak [mm]} & \textbf{Settle$<$2\,mm [s]} & \textbf{SS [mm]} & \textbf{Peak power [W]}\\
\midrule
No estimator   & 112.5         & never          & 111.2         & 6.1\\
With estimator & \textbf{13.7} & \textbf{2.70}  & \textbf{0.86} & 8.5\\
\bottomrule
\end{tabular}
\end{table}

Without the estimator, the response approaches (not fully, since the
reference moves along the circle while the push is one-directional and
$K_d$ acts on all three axes) the static equilibrium bias
$F_h/K_d\approx150$\,mm and never drops below the 2\,mm settling
tolerance: a permanent offset, as expected of any finite-stiffness law.
With the estimator, the same prescribed stiffness visibly yields to
contact --- a genuine 13.7\,mm transient, an order of magnitude more
compliant than C4/C5's few-mm response at their calibrated stiffness ---
and then recovers to 0.86\,mm within 2.70\,s: transient compliance
followed by offset-free recovery, extending the architecture beyond
disturbance rejection around a stiff equilibrium. Peak positive joint power stays low in both cases (6.1
and 8.5\,W) since the prescribed stiffness itself is low. Reproduced by
\texttt{low\_stiffness\_reference\_holding.py}.

\bibliographystyle{IEEEtran}
\bibliography{references}

\end{document}